\documentclass[journal]{IEEEtran}
\usepackage[caption=false,font=normalsize,labelfont=sf,textfont=sf]{subfig}
\usepackage{textcomp}
\usepackage{stfloats}
\usepackage{url}
\usepackage{verbatim}
\usepackage{graphicx}
\usepackage{cite}
\usepackage{xcolor}
\usepackage[unicode=true]{hyperref}
\usepackage{xspace}
\hypersetup{
	colorlinks=false,
	pdfborder={0.3 0.3 0.3},
	citecolor=green,
	linkcolor=blue,
	urlcolor=cyan
}
\usepackage{amsmath,amsfonts}
\usepackage{colortbl}
\usepackage{soul}
\definecolor{mixbg}{RGB}{235,237,245} 
\usepackage{amsmath,amsfonts}
\usepackage{graphicx}
\usepackage{epstopdf}
\usepackage{amsthm}
\usepackage{color}
\usepackage{algorithm}
\usepackage{algorithmic}
\usepackage{array}
\usepackage[normalem]{ulem}
\renewcommand{\algorithmicrequire}{\textbf{Input:}}
\renewcommand{\algorithmicensure}{\textbf{Output:}}
\usepackage{threeparttable}
\usepackage{amssymb}
\usepackage{mathrsfs}
\usepackage{diagbox}
\usepackage{booktabs}
\usepackage{multirow}
\usepackage{xcolor}
\usepackage{bm}
\usepackage{mdframed}
\newtheorem{theorem}{Theorem}
\newtheorem{lemma}{Lemma}
\newtheorem{remark}{Remark}

\newtheorem{assumption}{Assumption}

\newtheorem{definition}{Definition}
\newtheorem{principle}{Principle}

\newcommand{\name}{\texttt{CausShield}\xspace}

\usepackage{soul}

\begin{document}

\title{\name: Sample Reconstruction-Resilient Vertical FL via Causal Representation Learning}

\author{Yongqi~Jiang,
        Yansong~Gao,~\IEEEmembership{Senior Member,~IEEE,}
        Siguang~Chen,~\IEEEmembership{Senior Member,~IEEE,}
\textit{and}~Anmin~Fu,~\IEEEmembership{Member,~IEEE}

\thanks{This work is supported by the Fundamental Research Funds for the Central Universities (B250201047), the 333 High-level Talents Training Project of Jiangsu Province, the Open Research Fund of State Key Lab. for Novel Software Technology, Nanjing University (KFKT2025B03), and the National Natural Science Foundation of China (62372236). \emph{(Corresponding author: Siguang Chen.)}

Y. Jiang is with the School of Computer Science and Engineering, Nanjing University of Science and Technology, Nanjing 210094, China. (e-mail: jiangyongqi@njust.edu.cn)

Y.~Gao is with the University of Western Australia, Perth, Australia. (e-mail: gao.yansong@hotmail.com)

S. Chen is with the College of Computer Science and Software Engineering, Hohai University, Nanjing 211100, China, is also with the State Key Lab. for Novel Software Technology, Nanjing University, Nanjing 210023, China (E-mail: sgchen@hhu.edu.cn).

A.~Fu is with the School of Computer Science and Engineering, Nanjing University of Science and Technology, Nanjing 210094, China. (e-mail: fuam@njust.edu.cn)
}}

\IEEEtitleabstractindextext{
\begin{abstract}
Vertical federated learning (VFL) is a distributed learning paradigm that leverages vertically partitioned features across isolated parties without sharing raw sample features/data. 
It, however, remains vulnerable to active sample reconstruction attacks. Existing defenses against such attacks fail to achieve a satisfactory trade-off between model utility and privacy protection, due to either suppressing task-relevant information alongside privacy-sensitive features or relying on end-to-end supervised training to converge the defense module, which exposes the model to early-epoch vulnerability.

To address this challenge, we adopt a structural causal model (SCM) insight and construct \name. From a task-learning standpoint, causal features within a raw sample are those that are directly relevant and contributory to the learning objective, whereas non-causal features are task-irrelevant but often encode sample-specific private information, thereby facilitating reconstruction. Importantly, to lay a concrete foundation of \name, we theoretically prove this insight. \name thus decomposes the intermediate representations shared between the client and the coordinating server in VFL into task-relevant and task-irrelevant components to ensure full-cycle privacy protection.
Nonetheless, the decomposition is inherently challenging due to the dual objectives of preserving model utility while mitigating privacy leakage. We address this via a carefully formulated optimization problem, which is solved through unsupervised representation learning. We further theoretically prove that transmitting only causal representations can successfully resist sample reconstruction attacks while preserving the convergence behavior of standard VFL. Extensive experiments compare \name against seven state-of-the-art defenses, including InvL (USENIX Security'25), and evaluate robustness against advanced reconstruction attacks such as URVFL (NDSS'25). Results demonstrate that \name consistently outperforms prior defenses in terms of privacy protection, model utility, and computational efficiency.

\end{abstract}
	

}
\maketitle

\IEEEdisplaynontitleabstractindextext

\IEEEpeerreviewmaketitle

\section{Introduction}\label{sec:introduction}


\IEEEPARstart{V}ertically federated learning (VFL)~\cite{zhu2025passive} has attracted growing interest from academia and industry due to its ability to integrate diverse and complementary (non-overlapping) features, and it has been widely adopted in real-world applications, such as the FATE platform~\cite{Ye2025} and WeBank's federated risk control system~\cite{WeBank}. Meanwhile, sample-level reconstruction attacks~\cite{Erdogan2022,Yao2025} pose a severe privacy threat in VFL, where an honest-but-curious server can exploit the intermediate representations received during the training or inference stage to recover individual private inputs. Specifically, with the knowledge of ground-truth labels, the optimization-enabled attackers recover raw training samples by minimizing the distance between predictions and target representations through input adjustment~\cite{Erdogan2022}; the learning-based attackers train an inverse model to map intermediate representations back to raw inputs directly~\cite{ye2024feature}; the generative adversarial network (GAN)-based attackers employ a generator to synthesize inputs that approximate the distribution of the raw data~\cite{Yao2025}.


\noindent\textbf{Limitation of Existing Defenses.} To counter the threat of sample reconstruction, existing defenses can be broadly categorized into two categories: (1) traditional generic privacy-preserving methods and (2) learning-based methods. The first category includes secure multi-party computation (SMC)~\cite{khan2025secure}, homomorphic encryption (HE)~\cite{Zhu2024,li2025fedphe}, and differential privacy (DP)~\cite{cui2025aldp}---three foundational privacy-preserving techniques widely recognized as effective in both theoretical and practical contexts. However, they suffer significant limitations~\cite{hu2024sok}, drastically reducing utility or introducing additional computational and communication overheads, which consequently prolong the required training time considerably. Our work, therefore, avoids falling under this category.

\par The second category consists of a series of innovative learning-based methods including but not limited to lightweight perturbations~\cite{Gu2023,Xu2025} and representation remolding~\cite{Qiu2024,Singh2021,Li2022c,Noorbakhsh2024,Vepakomma2020,Chen2025Dual,Bi2024,shi2025defending}, such as adversarial representations~\cite{Singh2021}, imitating attacker actions~\cite{Li2022c}, and extending the irrelevance between inputs and their representations~\cite{Noorbakhsh2024,Vepakomma2020,Chen2025Dual}. 
Although these methods can mitigate the risk of data leakage to some extent, they are suffering from limitations \textbf{(Ls)}. 
\textbf{(L1)} Most of them ~\cite{Gu2023,Xu2025, Qiu2024,Singh2021,Li2022c,Noorbakhsh2024,Vepakomma2020} struggle to achieve a satisfactory trade-off between privacy preservation and model utility. The main reason is the lack of data decoupling capabilities; attempts to suppress sensitive information in the representations also indiscriminately suppress task-relevant information, leading to significant degradation in model utility. Among them, only works~\cite{Gu2023,Xu2025,Qiu2024} are designed for VFL.
\textbf{(L2)} We note recent methods~\cite{Bi2024,shi2025defending, Chen2025Dual} strike a balance between privacy-preservation and model utility; however, they cannot be directly extended to VFL, as they rely on label supervision or complete model parameters unavailable to VFL's passive parties (detailed in Section~\ref{sec:related_defense}). 



\par Aligned with~\cite{Luo2026,Xu2025,Yao2025}, by focusing on typical computer vision classification tasks, as the image modality is widely targeted and studied, we ask the following research question:

\begin{mdframed}[backgroundcolor=black!10,rightline=false,leftline=false,topline=false,bottomline=false,roundcorner=2mm]
   How to prevent the participant's raw sample reconstruction while preserving the global model utility of the VFL primary task?
\end{mdframed}

\noindent\textbf{Our Solution.} 
Previous works treat all information within representations equally of a given sample; therefore, they struggle to suppress privacy-related information without harming the task-relevant information, because the suppression is enforced indiscriminately.
In contrast, we propose \name, leveraging the structural causal model (SCM)~\cite{Scholkopf2022, Tople2020} to unsupervisedly decompose representations into causal (i.e., task-relevant) and non-causal (i.e., task-irrelevant) components, retaining only the former for VFL sharing while guaranteeing ultimate discarding of the latter. Generally speaking, SCM is a mathematical framework that uses a directed acyclic graph to explain the causal relationships between variables during the process of data generation. 

\par The core insight of \name lies in the fundamental distinction in how causal/non-causal components within the representation of a given sample contribute to the primary task and sample reconstruction. Intuitively, non-causal components (e.g., background and illumination) typically encapsulate rich sample-specific variations that distinguish individual instances from one another~\cite{Mahajan2021}. These variations contribute little to the primary task, but they provide highly identifiable statistical cues that adversaries can exploit for high-fidelity reconstruction, as we prove in Theorems in Section~\ref{theo_found}. Conversely, causal components capture invariant category-level features (e.g., semantic structures and shapes) that determine the class label~\cite{Mahajan2021}. Since samples of the same class share significant similarity in their task-relevant features, adversaries cannot distinguish specific instances based solely on such components, inherently limiting reconstruction fidelity.

We first lay a theoretical foundation that non-causal components dominate the primary information source for sample reconstruction (Theorem~\ref{prop:singular}), and transmitting only causal representations greatly raises the reconstruction error lower bound, thus fundamentally preventing sample reconstruction (Theorem~\ref{thm:invloss}). At the same time, these theorems indicate that non-causal components have a negligible contribution to the main task, which should therefore be discarded from a sample-level privacy-preservation perspective.

Guided by the theoretical foundation, \name transforms intermediate representations shared in VFL into purified causal representations that exclusively preserve task-relevant components while discarding task-irrelevant ones. Causal component decomposition is performed by the privacy-sensitive VFL participant before each round of representation transmission, ensuring full-cycle protection. 




\par Fulfilling \name is, however, non-trivial, which confronts a significant challenge (\textbf{C1}): how to disentangle and eliminate task-irrelevant information from representations while preserving only task-relevant components? The constraints of VFL further pose another critical challenge (\textbf{C2}). Unlike in centralized, horizontal FL and split FL settings, the passive parties (i.e., the defenders) in VFL have no access to labels or the full model parameters that are mandatory assumptions or knowledge by previous work~\cite{Bi2024}  (detailed in Section~\ref{sec:related_defense}). 



To this end, two properties that the task-relevant causal components in SCM must satisfy~\cite{Scholkopf2022, Tople2020, Lv2022}: (1) the changes on the task-irrelevant non-causal components do not affect the task-relevant ones; (2) the elements in the task-relevant causal representation should be mutually independent and decomposable. Directly solving this is hard. We instead perform causal representation extraction through two core consequential module designs, each of which is formulated as an optimization problem:  \textit{surrogate dataset generation module} (Section~\ref{sec:surrogateDataset}) synthesizes surrogate samples for each raw sample, ensuring significant randomization or variance in task-irrelevant components while assuring consistency in task-relevant components (Addressing \textbf{C1}); \textit{causal representation learning module} (Section~\ref{subsec IV.C}) facilitates mutual learning between raw and surrogate sample pairs in a fully unsupervised manner, requiring neither ground-truth labels nor any task-specific supervision (Addressing \textbf{C2}). 



Our main contributions are summarized as below:

\noindent$\bullet$ We are the first to introduce the privacy benefits of SCM into the VFL settings and, importantly, establish the theoretical foundations that prove transmitting only causal representations can fundamentally resist sample reconstruction attacks without harming model utility.


\noindent$\bullet$ We propose the \name guided by the theoretical foundation, which unsupervisedly remolds the raw representations of private inputs shared in VFL into causal representations that only capture task-relevant components. \name can be seamlessly integrated into standard VFL pipelines and strike the trade-off between privacy and model utility through two innovative module designs: surrogate dataset generation for synthesizing surrogate samples; causal representation learning for task-relevant component extraction.


\noindent$\bullet$ We provide the convergence proof of \name in standard VFL. Additionally, we conduct comprehensive experiments across five benchmark datasets and two mainstream attacks, with comparisons against seven SOTAs, including InvL (USENIX Security'25) and robustness evaluation against URVFL (NDSS'25). Results demonstrate that \name consistently outperforms prior defenses in privacy protection, model utility, and computational efficiency.



\par The rest of this article is organized as follows: Section~\ref{sec:2} outlines the relevant background knowledge and related work. Section~\ref{theo_found} presents \name's theoretical foundations. Section~\ref{sec:4} presents the threat model and details of \name. Subsequently, we provide the convergence proof in Section~\ref{sec:5}. The experimental evaluations are in Section~\ref{sec:6}. Finally, the conclusions are drawn in Section~\ref{sec:7}.

\section{Background and Related Work}
\label{sec:2}

\par This section provides the background of VFL, followed by related sample reconstruction attacks and defenses.

\subsection{Vertical Federated Learning}

\par VFL~\cite{li2023vertical,zhu2025passive} enables multiple participants to train a global model $F$, where each participant holds partial features, while the number of samples is the same among participants. In image-based VFL, there is typically one coordinating server (i.e., active party) and $K$ data owners (i.e., passive parties) that possess the same or similar image data samples yet are divided by feature dimensions. We define $D = \{ {D_k}\} _{k = 1}^K$ as the VFL private datasets and each dataset possesses $N$ training samples. The ${D_k} = \{ {x_{k,i}}\} _{i = 1}^N$ is local dataset that belongs to the $k$-th party. The corresponding labels are held exclusively by the active party, denoted as $\mathcal{Y} = \{ {y_{i}}\} _{i = 1}^N$, where $y_i$ is the label associated with the $i$-th sample across all parties.

\par \textbf{Passive parties:} These participants possess privacy-sensitive datasets. To mitigate the data silo issue, each passive party $k$ learns a local bottom model ${f_k}$ parameterized by ${\theta _k}$, which extracts an intermediate representation $\bm{\mathit{r}}_{k,i}$ from each local raw sample ${x_{k,i}}$. The final layer of ${f_k}$ is called the cut-layer. The intermediate representation $\bm{\mathit{r}}_{k,i}$ is uploaded to the active party. According to the chain rule, the model parameters of passive parties can be updated by backpropagating the cut-layer gradients ${\nabla _{{\theta _{{\rm{cut}}}}}}{\cal L}$ received from the active party.

\par \textbf{Active party:} The active party concatenates all the received representations $\bm{\mathit{r}}_i^s = [\bm{\mathit{ r}}_{1,i}^s,\bm{\mathit{ r}}_{2,i}^s...,\bm{\mathit{r}}_{K,i}^s]$ and then feeds them to the top model ${f_{\rm top}}$ parameterized by ${\theta _{\rm top}}$, which produces final predictions. Then the corresponding gradients are calculated and the top model is updated immediately. Next, the active party sends ${\nabla _{{\theta _{{\rm{cut}}}}}}{\cal L}$ to all passive parties.

The VFL can be formulated as:
\begin{equation}\label{eq1}
	\begin{split}
		\min_\Theta \mathcal{L}(\Theta) &= \frac{1}{N}\sum_{i=1}^N {l_{\rm ce}}(f_{\rm top}(f_1(\theta_1;x_{1,i}),\\
		&f_2(\theta_2;x_{2,i}), \ldots,f_K(\theta_K;x_{K,i}),y_i)),
	\end{split}
\end{equation}
where $\Theta  = \{ {\theta _1},{\theta _2},...,{\theta _K},{\theta _{\rm top}}\} $ represents the set of model parameters customized for parties. ${l_{\rm ce}}$ denotes the loss function i.e., Cross-Entropy (CE).

\subsection{Sample Reconstruction Attacks on Distributed Learning}

\par In the context of sample reconstruction against distributed systems (i.e., VFL and split FL), adversaries primarily exploit two distinct attack surfaces: (i) querying the bottom models of passive parties~\cite{luo2021feature,yang2023practical, he2019model}, and (ii) analyzing intermediate representations uploaded by passive parties during training~\cite{Pasquini2021,Erdogan2022,gao2023pcat,zhu2025passive,Yao2025}.

\par Early studies assumed attackers could query the target client's bottom model. In~\cite{luo2021feature,yang2023practical}, attackers leveraged known partial data and multiple query outputs to train generative models (i.e., GANs) that minimize the prediction discrepancy, thereby inferring the victim party's private data. Specifically,~\cite{luo2021feature} directly accesses model gradients, while~\cite{yang2023practical} employs zeroth-order optimization to estimate gradients without direct access. He \textit{et al.}~\cite{he2019model} find that attackers could reconstruct data with high fidelity using the intermediate representations, via techniques like regularized maximum likelihood estimation or inverse networks. Notably, these works impose strong, often impractical assumptions, requiring both access to real samples and ability to query the bottom model repeatedly.

\par As direct query access to victim models became increasingly restricted, recent works have explored alternative attack surfaces. FSHA~\cite{Pasquini2021} requires no knowledge of the client's model or data, using a shadow dataset and discriminator to guide representation alignment. Unsplit~\cite{Erdogan2022} assumes knowledge of the bottom model structure and observes uploaded representations without auxiliary data. PCAT~\cite{gao2023pcat} accesses auxiliary data but lacks bottom model knowledge, training a shadow model guided by the top model for reconstruction. However, these methods often suffer from limited reconstruction quality, particularly when the cut-layer is deep or the auxiliary data distribution shifts. More recently, SDAR~\cite{zhu2025passive} and URVFL~\cite{Yao2025} have achieved superior reconstruction performance. SDAR leverages adversarial regularization with a dual-discriminator mechanism to jointly train a shadow model and its inverse model. URVFL integrates a discriminator with an auxiliary classifier, using label information to guide the target model to mimic a pretrained encoder's embedding distribution for enhanced reconstruction.

\subsection{Defenses against Sample Reconstruction Attacks}\label{sec:related_defense}

Defense against sample reconstruction attacks falls into two categories: (i) traditional generic method~\cite{Mohassel2017,khan2025secure,Zhu2024,Lou2021,Qiu2024}, (ii) learning-based methods, such as lightweight perturbations~\cite{Geyer2017,cui2025aldp,Cao2023,Huang2020,Zou2022,Sun2021,Gu2023,Xu2025} and representation remolds~\cite{Qiu2024,Singh2021,Li2022c,Noorbakhsh2024}.

\par Traditional generic defenses against sample reconstruction include encryption-based methods such as SMC~\cite{Mohassel2017,khan2025secure}, HE~\cite{Zhu2024, Lou2021}, and hashing~\cite{Qiu2024}, which protect client privacy by encrypting or perturbing critical information during forward propagation. The work~\cite{Qiu2024} proposes a hash-based VFL framework that effectively eliminates the reversibility from representations back to data. These studies \cite{Zhu2024, Lou2021} allow for direct arithmetic operations, such as addition and multiplication, on data or its representations. Although these methods offer effective privacy protection, they typically incur significant computational and time overheads.

\par Lightweight perturbations mainly manifest in two forms: input perturbations and representation perturbations. The works~\cite{Geyer2017,cui2025aldp,Cao2023,Huang2020,Zou2022} utilize generative models or introduce noise to create distorted pseudo-synthetic data, which then replace the raw training data to safeguard privacy. Soteria~\cite{Sun2021} estimates and prunes the most privacy-beneficial representation elements. Gu \textit{et al.} embeds a private passport by applying scale and shift factors to the batch normalization (BN) layers, thereby perturbing representations in the forward propagation process~\cite{Gu2023}. More recently, Xu \textit{et al.}~\cite{Xu2025} propose InvL-based adaptive noise perturbation strategies that inject noise calibrated by the spectral properties of the Jacobian matrix, achieving a favorable trade-off between privacy protection and model utility compared to conventional noise-based defenses. Despite the seemingly impregnable design of these methods, Li \textit{et al.}~\cite{Li2022b} find that when adversaries possess enough prior knowledge to make up for the missing details, the defensive efficacy of these techniques is dramatically reduced.

\begin{table}[!t]
\centering
\caption{Comparison with representative learning-based defense methods.}
\label{tab:comparison}
\resizebox{\linewidth}{!}{
\begin{threeparttable}
\begin{tabular}{lccccc}
\toprule
\textbf{Method} & \textbf{Label-free} & 
\textbf{Pre-deployable} & \textbf{Privacy-Utility} & 
\textbf{VFL Compatible} & \textbf{Theo.\tnote{1} Guarantee} \\
\midrule
Soteria~\cite{Sun2021}         & \checkmark & \checkmark & $\sim$      & \checkmark & \checkmark \\
NoPeek~\cite{Vepakomma2020}    & \checkmark & \texttimes  & $\sim$      & \checkmark & \texttimes \\
DISCO~\cite{Singh2021}         & \checkmark & \texttimes  & $\sim$      & \checkmark & \texttimes \\
ResSFL~\cite{Li2022c}          & \texttimes  & \texttimes  & $\sim$      & $\sim$     & \texttimes \\
DMIAFP~\cite{shi2025defending} & \texttimes  & \texttimes  & \checkmark  & \texttimes & \texttimes \\
D3SFL~\cite{Chen2025Dual}      & \texttimes  & \texttimes  & \checkmark  & $\sim$ & \texttimes \\
PriFU~\cite{Bi2024}            & \texttimes  & \texttimes  & \checkmark  & \texttimes & \texttimes \\
Inf2Guard~\cite{Noorbakhsh2024}& \texttimes  & \texttimes  & \checkmark      & $\sim$     & \checkmark \\
InvL~\cite{Xu2025}             & \checkmark & \checkmark & $\sim$      & \checkmark & \checkmark \\
\midrule
\rowcolor{blue!5}\textbf{\name}          & \checkmark & \checkmark & \checkmark  & \checkmark & \checkmark \\
\bottomrule
\end{tabular}
\begin{tablenotes}
\footnotesize
\item[1] Theo. -- Theoretical 
\centering
\item \checkmark: supported; \texttimes: not supported; 
$\sim$: partially supported.
\end{tablenotes}
\end{threeparttable}}
\end{table}

\par Representation remolds are designed to decrease the correlation between the intermediate representations and raw data. Qiu \textit{et al.}~\cite{Qiu2024} develop a framework based on adversarial training. It consists of three modules: adversarial reconstruction, noise regularization, and distance correlation minimization. Because these modules function independently of each other, they can be used singly or in combination. DISCO~\cite{Singh2021} learns a dynamic and data-driven filter to selectively obscure sensitive information within the representation space. ResSFL~\cite{Li2022c} employs attacker-aware training to obtain a resistant feature extractor, which is then used to initialize clients' parameters before each training round to avoid data leakage during training. Inf\textsuperscript{2}Guard~\cite{Noorbakhsh2024} proposes a unified information-theoretic defense framework that minimizes the mutual information between representations and private inputs to preserve privacy, while maximizing the mutual information between representations and labels to maintain utility.

\par Although these methods~\cite{Noorbakhsh2024,Singh2021,Li2022c,Qiu2024} can achieve privacy protection to a certain extent, they treat all information in inputs or representations equally without distinguishing task-relevant from task-irrelevant components or representations, causing them to indiscriminately suppress both types of information and thus inevitably incur non-negligible model utility degradation. To deal with these limitations, some methods~\cite{Bi2024, shi2025defending, Chen2025Dual} seek to preserve task-relevant information while suppressing privacy-sensitive information but fail to be directly extended to VFL due to their reliance on label supervision or complete model parameters. Specifically, DMIAFP~\cite{shi2025defending} reduces both first-order and second-order correlations between private inputs and representations through feature purification, but requires ground-truth labels and a classifier to train its purification module jointly. D3SFL~\cite{Chen2025Dual} introduces variable-structure sub-networks with sparse binary masks to disentangle input-feature dependencies, yet it relies on a locally simulated inversion model trained with a labeled auxiliary dataset and complete model gradients. PriFU~\cite{Bi2024} leverages gradient magnitude to identify and prune task-irrelevant dimensions, yet requires ground-truth labels and end-to-end backpropagation, and can only protect privacy during inference, leaving raw representations exposed throughout training.

These incompatibilities prevent a direct adoption of label-supervised defense in VFL, motivating a novel defense that operates without label supervision and provides full-cycle protection from the first training round. TABLE~\ref{tab:comparison} summarizes the key properties of representative defenses and highlights the limitations that motivate our work.

\section{Theoretical Foundations: Causal Representations Resist Sample Reconstruction}\label{theo_found}


\par While both the causal and non-causal components coexist in the raw representation, we establish a theoretical framework comprising two progressive claims, supported by empirical evidence, to prove that attackers \textit{cannot distinguish specific instances} based solely on the causal representation. First, we prove that the non-causal components of the shared representation are the primary source for high-fidelity sample reconstruction (Theorem~\ref{prop:singular}). Building on it, we demonstrate that transmitting only the causal representation can fundamentally resist sample reconstruction (Theorem~\ref{thm:invloss}).

\par \textbf{Structural Causal Model.} In the SCM~\cite{Scholkopf2022,Tople2020}, each raw input $x$ comprises a blend of causal factors $S$ (task-relevant) and non-causal factors $U$ (task-irrelevant), i.e., $P(x|S, U)$. The causal factors $S$ uniquely determine the class label $y$, i.e., $P(y|S)$, while non-causal factors $U$ encode rich instance-specific variation. While~\cite{Tople2020} leverages causal learning to resist membership inference attacks in centralized settings, our work leverages the privacy benefits of SCM to defend against a different sample reconstruction attack in VFL, operating without label supervision. The SCM construction principles are below:

\begin{principle}[Common Cause Principle~\cite{peters2017,Scholkopf2022}]\label{principle1}
\par When two variables, $X$ and $Y$, exhibit statistical dependence, it indicates the presence of a variable $S$ that causally influences both of them and accounts for the entirety of their interdependence.
\end{principle}

\begin{principle}[Independent Causal Mechanisms 
(ICM) Principle~\cite{Scholkopf2022}]\label{principle2}
The elements in causal factor $\{s_1, s_2, \ldots\} \in S$ are mutually independent: changes on $P(s_i|B_i)$ do not affect $P(s_j|B_j)$ for $i \neq j$, and their joint distribution factorizes as $P(s_1, s_2, \ldots) = \prod_{i=1}^{|S|} P(s_i|B_i)$.
\end{principle}

\par What Principles~\ref{principle1}-\ref{principle2} declare is that elements $\{{s_1},{s_2},...\}$ in causal factors should be jointly independent, and the knowledge contained in elements is not repeated~\cite{Lv2022}.

\begin{mdframed}[backgroundcolor=black!10,rightline=false,leftline=false,topline=false,bottomline=false,roundcorner=2mm]
\subsubsection*{\textbf{Key Insight~1}}
\textit{Non-causal components dominate the singular value directions of the shared representation's Jacobian, making them the primary information source for sample reconstruction attacks.}
\end{mdframed}

To formalize, we first introduce two assumptions grounded in the SCM structure, and then define the representation's Jacobian decomposition.


\begin{assumption}[Non-Causal Instance-Specificity]
\label{assump:noncausal}
Non-causal factors $U$ encode rich instance-specific variation 
beyond class labels, i.e., $I(U;x|y)>0$. This implies that 
the non-causal representation $\bm{r}^u$ must be sufficiently 
sensitive to $U$ across all top-$k$ directions to preserve 
such variation:
\begin{equation}
    \exists\; \epsilon > 0, \quad 
    \sigma_k\!\left(\frac{\partial \bm{r}^u}{\partial U}\right) 
    \geq \epsilon,
\end{equation}
where $\sigma_k(\cdot)$ denotes the $k$-th 
largest singular value measuring the sensitivity of the mapping along each direction.
\end{assumption}

\begin{assumption}[Causal Class-Level Invariance]
\label{assump:causal}
Causal factors $S$ capture class-level invariant features. For any 
two instances $i,j$ with $y_i = y_j$, the invariance condition 
$\bm{r}^s_i \approx \bm{r}^s_j$ implies that $\bm{r}^s$ changes slowly along 
$S$-directions within a class:
\begin{equation}
    \left\|\frac{\partial \bm{r}^s}{\partial S}\right\|_2 \leq \gamma, 
    \quad \text{where } \gamma \ll \epsilon.
\end{equation}
\end{assumption}

\begin{definition}[SCM-Induced Jacobian Decomposition]
\label{def:jacobian}
Let $\bm{r} = f_{\rm bottom}(x) \in \mathbb{R}^{1 
\times J}$ be the intermediate representation shared in VFL, and let $G_x = \partial \bm{r}/\partial x $ be its Jacobian matrix. Under the SCM decomposition $\bm{r} = \bm{r}^s + \bm{r}^u$, the Jacobian decomposes as:
\begin{equation}
    G_x = G^S_x + G^U_x, \quad
    G^S_x = \frac{\partial \bm{r}^s}{\partial x}, 
    \quad
    G^U_x = \frac{\partial \bm{r}^u}{\partial x},
\end{equation}
where $S \perp U$ by the SCM ensures orthogonal row spaces $\mathcal{V}$: $\mathrm{row}(G^S_x) \subseteq \mathcal{V}^S$ 
and $\mathrm{row}(G^U_x) \subseteq \mathcal{V}^U$, with 
$\mathcal{V}^S \perp \mathcal{V}^U$; $\sigma_i(\cdot)$ denotes the $i$-th largest singular value of a matrix, with $\sigma_1(\cdot)$ being the spectral norm;.
\end{definition}

\begin{theorem}[Non-Causal Components Dominate Reconstruction Risk]
\label{prop:singular}
Under Assumptions~\ref{assump:noncausal}-\ref{assump:causal} and Definition~\ref{def:jacobian}, for all $i \in \{1, \ldots, k\}$, the $i$-th largest singular values of the causal and non-causal Jacobians satisfy:
\begin{equation}
    \sigma_i(G^U_x) \gg \sigma_i(G^S_x),
\end{equation}
i.e., the top-$k$ singular directions of $G_x$ are dominated by the non-causal Jacobian $G^U_x$, making non-causal components the primary information source for sample reconstruction.
\end{theorem}

\begin{proof}
See Appendix~\ref{append:A} for details. The theoretical foundation of Theorem~\ref{prop:singular} builds on the work~\cite{Xu2025}, which established that the top-$k$ largest singular values of the representation's Jacobian matrix determine the directions most exploitable for reconstruction. In other words, larger singular values correspond to directions along which the input-to-representation mapping is most invertible, providing adversaries with more recoverable information and thus higher reconstruction risk. However, \cite{Xu2025} does not analyze which components within the representation (e.g., causal or non-causal) dominate these singular value directions and reconstruction risk. Grounded in this gap, we theoretically prove that the top-$k$ singular values of the non-causal Jacobian $G^U_x$ strictly dominate those of 
the causal Jacobian $G^S_x$, i.e., $\sigma_i(G^U_x) \gg \sigma_i(G^S_x)$ for all $i \in \{1,\ldots,k\}$, confirming that non-causal components constitute the dominant reconstruction attack surface. 
\end{proof}

\begin{mdframed}[backgroundcolor=black!10,rightline=false,leftline=false,topline=false,bottomline=false,roundcorner=2mm]
\subsubsection*{\textbf{Key Insight~2}}
\textit{Sharing causal representations can raise the lower bound of reconstruction error, rendering sample reconstruction fundamentally infeasible.}
\end{mdframed}

The Invertibility Loss 
(i.e., $\mathrm{InvLoss}$) defined in~\cite{Xu2025} quantifies the minimum reconstruction error achievable by the attacker for a given input $x$ from shared raw representation $\bm{r}$, which for a rank-$k$ optimal attacker is bounded as:
\begin{equation}
\label{eq:invloss}
    \mathrm{InvLoss}_x = \min_{A_x} 
    \|A_x(\bm{r}) - x\|^2 \leq
    \sum_{i=k+1}^{d}(V_{J,i}^\top x)^2 + C_0,
\end{equation}
where $A_x$ denotes the optimal inverse transformation; $G_x = U_J\Sigma V_J^\top$ is the SVD of the Jacobian; $V_{J,i}$ is the $i$-th right singular vector corresponding to the $i$-th largest singular value; and $C_0 = o(\delta^2)$ is the higher-order remainder from local linearization, treated as a constant. A larger residual tail energy $\sum_{i=k+1}^{d}(V_{J,i}^\top x)^2$ corresponds to a higher $\mathrm{InvLoss}$, meaning less information about $x$ is recoverable from $\bm{r}$ and the attacker's reconstruction is provably 
inaccurate.

\begin{theorem}[Causal Representations Resist Sample Reconstruction]
\label{thm:invloss}
Let $\tau_{\mathrm{th}} := \sum_{i=k+1}^{d}(V_{J,i}^\top x)^2 
    + \sum_{i=1}^{k}
    \frac{(U_i^\top \varepsilon)^2}{\sigma_i^2 \cdot p} 
    + C_0$ denote the defense success threshold, where a defense succeeds if and only if it raises $\mathrm{InvLoss}$ strictly above $\tau_{\mathrm{th}}$. Transmitting only the causal representation $\bm{r}^s$ yields a reconstruction error lower bound strictly and substantially exceeds $\tau_{\mathrm{th}}$:
\begin{equation}
\label{eq:lower_bound}
    \mathrm{InvLoss}(\bm{r}^s) \geq 
    (1 - k\gamma^2)\|x^s\|^2 + C_0 \gg \tau_{\mathrm{th}}.
\end{equation}
\end{theorem}

\begin{proof}
See Appendix~\ref{append:B} for details. \cite{Xu2025} proves the InvLoss upper bound for noise-based representation perturbation, showing that such perturbation can resist sample reconstruction albeit at the cost of representation utility. We adopt this upper bound as the defense success threshold $\tau_{\mathrm{th}}$. Building on this, Theorem~\ref{thm:invloss} further proves that transmitting only causal representations yields $\mathrm{InvLoss}(\bm{r}^s)$ far exceeding the defense success threshold $\tau_{\mathrm{th}}$. Intuitively, this result stems from the fundamental roles that causal components play in both the primary task and data reconstruction, capturing invariant features shared across instances of the same class~\cite{Mahajan2021}. Since samples from the same class exhibit significant similarity in their task-relevant features, adversaries cannot distinguish specific instances based solely on these features.
\end{proof}

\begin{figure}[htbp]
    \centering
    \subfloat[\scriptsize InvRE($\downarrow$)]{\includegraphics[width=0.23\textwidth]{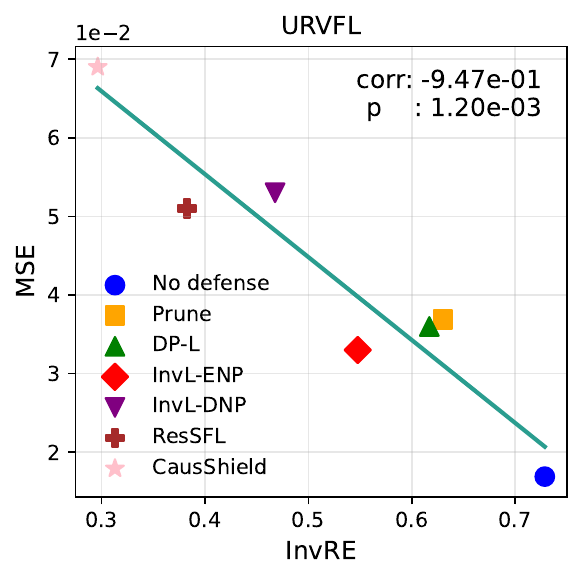}\label{fig:invre}}
    \subfloat[\scriptsize Cosine similarity($\uparrow$)]{\includegraphics[width=0.23\textwidth]{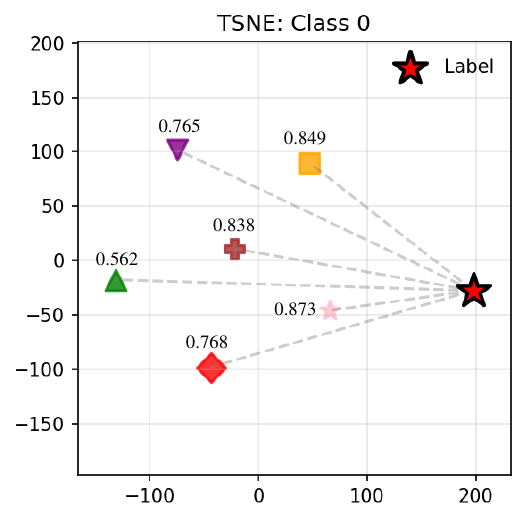}\label{fig:cosine}}
    \caption{Evidence of reconstruction risk and task relevance.}
    \label{fig:combined}
\end{figure}

\noindent\textbf{Empirical Evidence:} Fig.~\ref{fig:combined} jointly validates both Theorems~\ref{prop:singular}-\ref{thm:invloss}. Fig.~\ref{fig:combined}(\subref*{fig:invre}) shows the correlation between the InvRE score (which measures the reconstruction risk of shared representations derived from InvLoss) and the reconstruction mean square error (MSE) on CIFAR-10. A lower InvRE means less invertible information. Among all compared methods, \name achieves the lowest InvRE ($0.296$) and the highest reconstruction MSE ($0.069$). This confirms that transmitting only causal representations minimizes information exploitable to attackers and resists reconstruction attacks, validating Theorem~\ref{thm:invloss}. In contrast, when no defense applied, the full representation $\bm{r} = \bm{r}^s + \bm{r}^u$ contains both causal and non-causal components, yielding the highest InvRE ($0.729$) and the lowest MSE ($0.017$). Since the full representation can be decomposed into $\bm{r}^s$ and $\bm{r}^u$, the substantially higher reconstruction risk compared to \name confirms that non-causal components $\bm{r}^u$ constitute the primary information source exploited by the attacker, providing empirical support for Theorem~\ref{prop:singular}.

Fig.~\ref{fig:combined}(\subref*{fig:cosine}) evaluates the task-relevance of representations by measuring their cosine similarity to a causal proxy. Specifically, we select samples that maintain stable neighborhood relationships across multiple convolutional layers with high classification confidence as proxies for task-relevant causal information; higher cosine similarity indicates that the transmitted representation preserves more task-relevant causal information. As shown, \name achieves the highest cosine similarity to the causal proxy ($0.873$), outperforming ResSFL ($0.838$), InvL-ENP ($0.768$), Prune ($0.765$), and InvL-DNP ($0.562$). This demonstrates that while \name aggressively suppresses non-causal components to minimize reconstruction risk, it simultaneously preserves causal components most faithfully, achieving a superior privacy-utility trade-off over existing defense methods.

\section{\name}
\label{sec:4}

\par We define the threat model and then overview the \name and elaborate on its two core module designs.

\subsection{Threat Model}


\par \textit{1) Attacker Capability and Goal:} We adopt the widely used honest-but-curious assumption for the active party. The attack surface in our setting aligns with the standard assumptions in VFL sample reconstruction works~\cite{Pasquini2021,Erdogan2022,Yao2025}. In this setting, the active party faithfully adheres to the federated protocol but attempts to infer the passive party's private training data. As the central adversary, the active party controls the top model $f_{\rm top}$ and has full access to the ground-truth labels $\mathcal{Y}$. Furthermore, the adversary is assumed to be aware of the passive party's local model architecture and can intercept the intermediate representations shared during training. Leveraging knowledge of these representations alongside passive party's model structure, the active party aims to reconstruct the corresponding private data~\cite{Erdogan2022}. To further enhance the fidelity of the reconstructed data, the adversary can also possess an auxiliary dataset that follows a similar distribution to the training data~\cite{Yao2025}. Each sample in this auxiliary dataset contains a complete set of features and its corresponding label. 

\par \textit{2) Defender Capability and Goal:} The defense objective of passive parties is to reshape the intermediate representations before sharing, reducing the exposure of private information while preserving representation utility. The passive party has full knowledge of its local training data, the architecture and parameters of its bottom model, and the entire training process. Meanwhile, the passive party strictly protects local data privacy by prohibiting any form of direct data query. Passive parties cannot communicate with each other, nor can they access any auxiliary datasets or additional label information. We assume that the passive party possesses relevant but minimal machine learning knowledge, enabling it to understand the attack mechanism and implement defense strategies.

\subsection{Overview}\label{subsec IV.overview}

\par Fig.~\ref{Fig.3} illustrates the workflow of our proposed \name, which comprises two unsupervised modules: \textit{surrogate data generation} and \textit{causal representation learning}, addressing \textbf{C1} and \textbf{C2} respectively. The former module is performed \textit{offline} before VFL training; the latter is executed prior to each round of representation transmission, thereby providing full-cycle privacy protection.

\begin{figure*}[!t]
	\centering
	\includegraphics[width=6.5in]{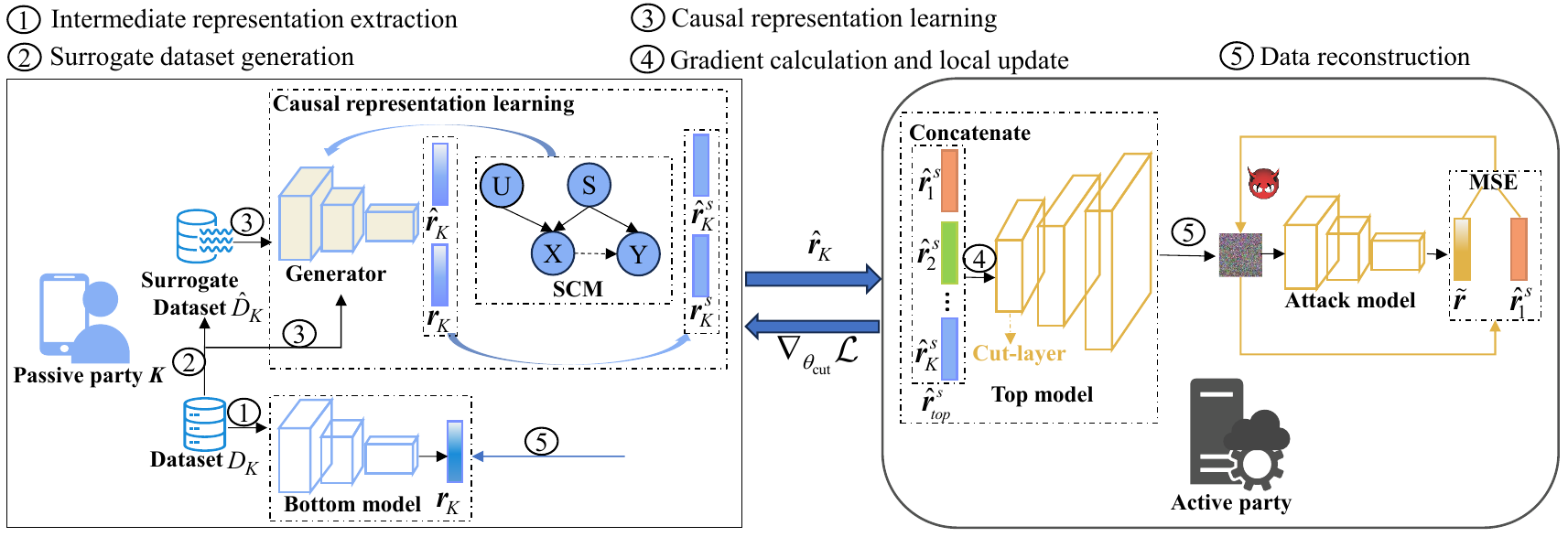}
	\caption{\name overview.}
	\label{Fig.3}
\end{figure*}

\par The \textbf{surrogate dataset generation module} synthesizes a surrogate input $\hat{x}_{k,i}$ for each raw input $x_{k,i}$ in the passive party's local dataset $D_k$, yielding a surrogate dataset $\hat{D}_k$. The surrogate $\hat{x}_{k,i}$ is constructed to satisfy two requirements: it diverges from $x_{k,i}$ in task-irrelevant (i.e., non-causal) components $U$ (e.g., color), while remaining consistent in task-relevant (i.e., causal) components $S$ (e.g., semantic structure, shape). Such surrogate pairs $(x_{k,i}, \hat{x}_{k,i})$ provide the contrastive signal necessary for causal representation learning without requiring any label supervision.

\par The \textbf{causal representation learning module} leverages the surrogate pairs to distill causal representations in a fully unsupervised manner. Specifically, the passive party $k$ duplicates its bottom model as a generator $G_k$, and feeds both $x_{k,i}$ and $\hat{x}_{k,i}$ into $G_k$ to obtain the raw representation $\bm{r}_{k,i}$ and the surrogate representation $\hat{\bm{r}}_{k,i}$. Guided by Principles~\ref{principle1}-\ref{principle2}, the generator is then optimized to transform $\bm{r}_{k,i}$ and $\hat{\bm{r}}_{k,i}$ into causal representations $\bm{r}^s_{k,i}$ and $\hat{\bm{r}}^s_{k,i}$ that satisfy two properties: (1) same-dimension correlations between $\bm{r}^s_{k,i}$ and $\hat{\bm{r}}^s_{k,i}$ are maximized, ensuring invariance to non-causal changes; (2) cross-dimension correlations are minimized, ensuring mutual independence among causal factors. The causal representation ${\bm{r}}^s_{k,i}$ is then uploaded to the active party for joint prediction, while the bottom model parameters are updated via the cut-layer gradients $\nabla_{\theta_\mathrm{cut}}\mathcal{L}$ backpropagated from the active party.

\subsection{Surrogate Dataset Generation}\label{sec:surrogateDataset}

\par To address \textbf{C1}, we design this module to synthesize a surrogate for each raw input that differs greatly in non-causal components while preserving causal ones. The motivation lies in the invariant causal mechanisms~\cite{peters2017}: changes in task-irrelevant features (e.g., color) alter only visual appearance without influencing $P(y|S)$, such that surrogate data can significantly facilitate causal representation learning~\cite{Ahuja2023}. 

\par As shown in Fig.~\ref{Fig.4}, we decompose the RGB image $x$ into luminance $x_L$ and chrominance $x_{ab}$ in Lab color space (also known as CIELAB color space), and design a dual-branch encoder consisting of two subnetworks $A_1$ and $A_2$ with two core designs: \textit{(i) Altering Non-Causal Components}: $A_1$ generates chrominance output $\hat{x}_{ab}$ with substantially different appearance from the original $x_{ab}$, by adopting CE loss over quantized color distributions and window-based variance regularization to disrupt both global and local color patterns. \textit{(ii) Preserving Causal Components}: $A_2$ reconstructs luminance 
output $\hat{x}_L$ that faithfully retains the structural content of the original $x_L$, ensuring causal consistency between $\hat{x}$ and $x$. Upon the completion of the dual-branch encoder training process, ${\hat x_L}$ and ${\hat x_{ab}}$ are consolidated into the surrogate data $\hat x = \{ {\hat x_L},{\hat x_{ab}}\}$.

\begin{figure}[htb]
	\centering
	\includegraphics[width=3.2in]{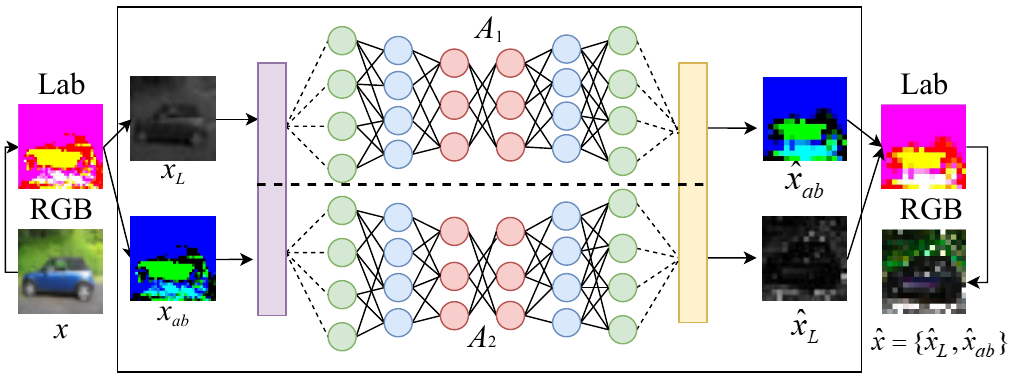}
	\caption{Surrogate data generation paradigm.}
	\label{Fig.4}
\end{figure}



\noindent \textbf{Altering Non-Causal Components.} 
Conventional colorization methods~\cite{Zhang2017,kang2023ddcolor} minimize pixel-level distance loss (e.g., Euclidean or $L_1$ norm), driving the subnetwork to faithfully reproduce exact pixel values and producing $\hat{x}_{ab}$ that remains chromatically close to $x_{ab}$, thus failing to alter non-causal components sufficiently. To address this, we introduce the following two designs for $A_1$ that responsible for automatic colorization ($x_L \rightarrow \hat x_{ab}$): 

\textit{CE Loss over Quantized Color Distributions.} We adopt CE loss $l^{ce}$ over quantized color distributions to replace pixel-level distance loss, enabling $A_1$ to learn diverse color statistics rather than precise pixel values and generating $\hat{x}_{ab}$ with a substantially different global color distribution from $x_{ab}$:
\begin{flalign} \label{eq7}
    {A_1}^* &= \mathop {\arg \min }\limits_{{A_1}} 
    l^{ce}({A_1}({x_L}),{{x}_{ab}}) \nonumber &\\
    &= \mathop {\arg \min }\limits_{{A_1}} 
    - \sum\limits_{h,w} {\sum\limits_q 
    ({x_{ab}})_{h,w,q}} 
    \log ({A_1({x_L})_{h,w,q}}).  &
\end{flalign}


\textit{Window-Based Variance 
Regularization.} Although CE loss for $A_1$ alters the global color distribution of $\hat{x}_{ab}$, local color patterns within small spatial regions may still resemble those of $x$, providing identifiable cues that adversaries could exploit for reconstruction~\cite{Luo2025}. To this end, we augment $A_1$'s objective with a window-based variance regularization term. By dividing 
$\hat{x}_{ab}$ into $n$ non-overlapping windows of size $s \times s$ and penalizing deviations from a target variance $v_e$, we actively disrupt local color consistency between $\hat{x}_{ab}$ and $x_{ab}$:
\begin{equation} \label{eq9}
	\begin{split}
		\mathop {\arg \min }\limits_{{A_1}^*} & - \sum\limits_{h,w} {\sum\limits_q {{{({{x}_{ab}})}_{h,w,q}}} } \log {(({\hat x_{ab}})_{h,w,q}})   \\
		&+ \mu \sum\limits_i {(Var{{(wi{n_i}({{\hat x}_{ab}}) - {v_e})}^2}},
	\end{split}
\end{equation}
where $\mathrm{Var}(\cdot)$ computes the pixel variance within each window and $v_e$ governs the expected degree of local color variation.

\noindent \textbf{Preserving Causal Components Consistency.} For $A_2$ (tasked with grayscale prediction ($x_{ab} \rightarrow \hat x_L$)), since luminance encodes structural information (i.e., object shapes and edges) that constitutes causal components, we minimize the Euclidean distance $l^{Euc}$ to ensure $\hat{x}_L$ faithfully preserves the structural characteristics of $x_L$, maintaining causal consistency between $\hat{x}$ and $x$:
\begin{equation} \label{eq8}
	\begin{aligned}
		{A_2}^* &= \mathop {\arg \min }\limits_{{A_2}} l^{Euc}({A_2}({x_{ab}}),{x_L})\\
		&= \mathop{\arg\min}_{{A_2}} 
\sum\limits_{h,w}\|{A_2}({x_{ab}})_{h,w} - 
({x_L})_{h,w}\|^2.
	\end{aligned}
\end{equation}

\par Overall, the final surrogate data $\hat{x} = \{\hat{x}_L, \hat{x}_{ab}\}$ serves as a high-quality contrastive signal for the subsequent causal representation learning module.

\subsection{Causal Representation Learning}\label{subsec IV.C}




\par To address \textbf{C2}, we design this module to unsupervisedly learn causal representation solely with the aid of surrogate representation ${\bm{\mathit{\hat r}}_{k,i}} = {G_k}({\hat x_{k,i}})$. Guided by the Principles~\ref{principle1}-\ref{principle2}, we first define two properties that must be satisfied to obtain accurate causal representations $\bm{r}^s_{k,i}$:

\par \textbf{Property 1.} Changes on the non-causal components $U$ do not affect the causal representation $\bm{\mathit{r}}_{k,i}^s$.
\par \textbf{Property 2.} The elements $\{r_{k,i}^{s,1},r_{k,i}^{s,2},...,r_{k,i}^{s,J}\}$ of causal representation $\bm{\mathit{r}}_{k,i}^s$ are jointly independent and decomposable.

\par As shown in Fig.~\ref{Fig.5}, to satisfy the above two properties while preserving representation utility, we introduce two core designs: \textit{(i) Causal Representation Extraction}: we measure the correlations between raw and surrogate representations across all dimensions via a correlation matrix $\mathbf{C}$, and design a decomposition loss $l_{Dec}$ to jointly maximize same-dimension correlations and minimize cross-dimension correlations, extracting causal components efficiently; \textit{(ii) Task Relevance Augmentation}: since not all causal dimensions carry sufficient task-relevant information, we design a masker $\bm{\omega}$ to evaluate each dimension’s contribution and distinguish high-contribution upper dimensions from low-contribution lower ones. Through adversarial training between the generator $G_k$ and the masker $\bm{\omega}$, lower dimensions are progressively enriched with additional task-relevant information, improving the overall utility of the causal representations.


\begin{figure*}[htb]
	\centering
	\includegraphics[width=5.5in]{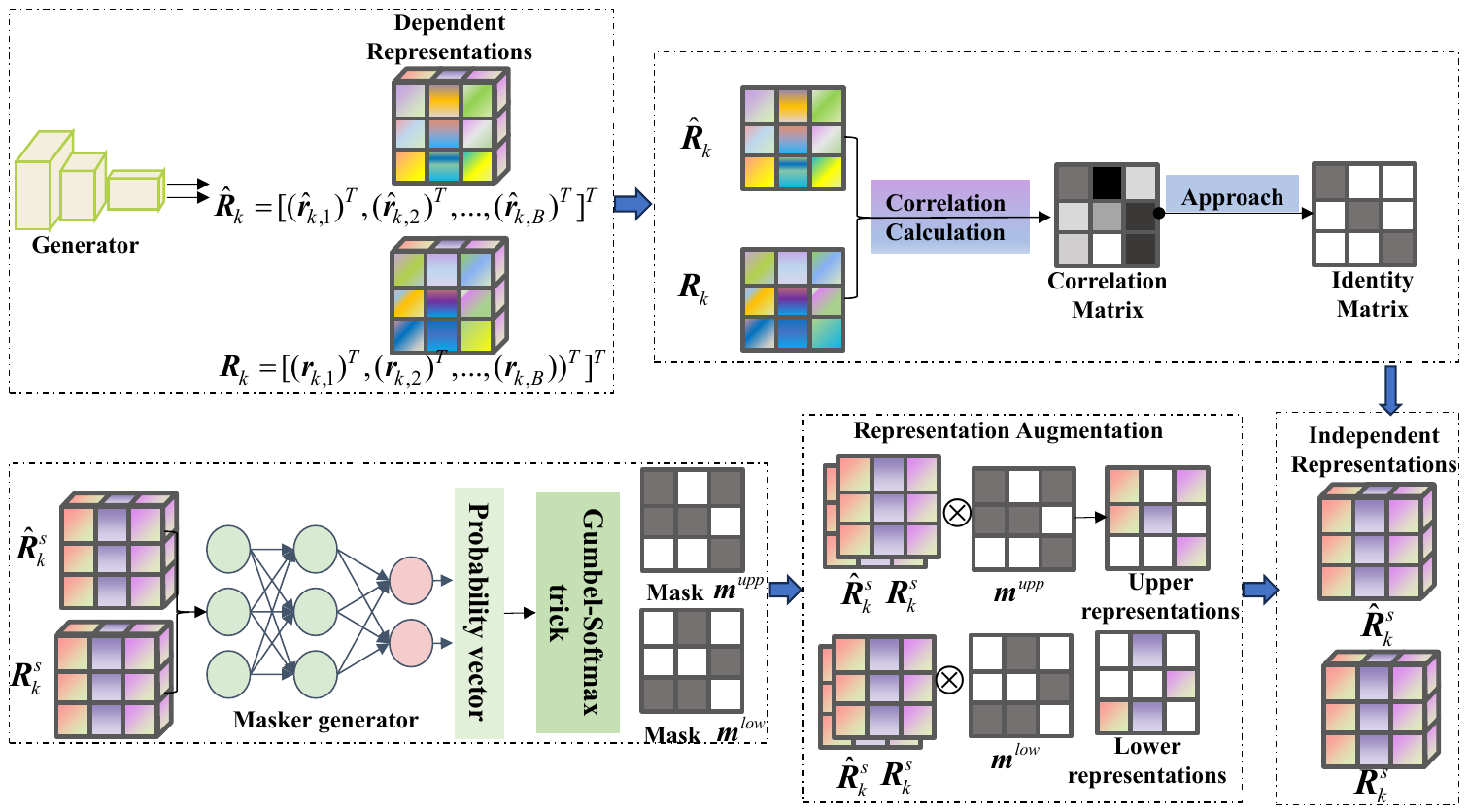}
	\caption{Causal representation learning process.}
	\label{Fig.5}
\end{figure*}

\noindent \textbf{Causal Representation Extraction.} We employ a cross-dimensional correlation function (denoted as \textit{COR}) to quantify the correlation between raw and surrogate representation pairs ($\bm{r}_{k,i}$, $\hat{\bm{r}}_{k,i}$), which are Z-score normalized for fair assessment. 

\par To retain dimensions invariant to non-causal changes (\textit{satisfying Property 1}), we optimize ${G_k}$ to maximize the same-dimension correlation:
\begin{equation} \label{eq12}
	\mathop {\max }\limits_{{G_k}} \frac{1}{J}\sum\limits_{j = 1}^J {COR(r_{k,i}^j,\hat r_{k,i}^j)} ,
\end{equation}
where $r_{k,i}^j$ and $\hat r_{k,i}^j$ represent the Z-score normalized $j$-th column of the matrix ${{\bf{R}}_k} = {[{({\bm{\mathit{r}}_{k,1}})^T},{({\bm{\mathit{r}}_{k,2}})^T},...,{({\bm{\mathit{r}}_{k,B}})^T}]^T}$ and ${{\bf{\hat R}}_k} = {[{({\bm{\mathit{\hat r}}_{k,1}})^T},{({\bm{\mathit{\hat r}}_{k,2}})^T},...,{({\bm{\mathit{\hat r}}_{k,B}})^T}]^T}$, respectively; $B$ denotes the current batch size and ${{\bf{R}}_k} \in \mathbb{R}{^{B \times J}}$.

\par To ensure each dimension captures independent and non-redundant causal information (\textit{satisfying Property 2}), we aloptimize ${G_k}$minimize cross-dimensional correlations:
\begin{equation} \label{eq13}
	\min_{G_k} \frac{1}{J(J - 1)} \sum_{{j_1 \neq j_2}} \textit{COR}(r_{k,i}^{j_1},\hat r_{k,i}^{j_2}), \quad j_1 \neq j_2.
\end{equation}

\par Note that optimizing Formulas~(\ref{eq12}) and~(\ref{eq13}) separately requires two independent passes over all $J(J-1)$ dimension pairs, incurring quadratic computational cost $\mathcal{O}(J^2)$. To address this efficiency bottleneck, we consolidate them into a single cross-dimensional correlation matrix $\mathbf{C}$:


\begin{equation} \label{eq14}
	\mathbf{C}_{j_1,j_2} = \frac{\langle r_{k,i}^{j_1},\hat{r}_{k,i}^{j_2}\rangle }{||r_{k,i}^{j_1}||\,||\hat{r}_{k,i}^{j_2}||}, \quad j_1, j_2 \in \{1,2,\ldots,J\},
\end{equation}
where the inner product operation is indicated by $\langle  \cdot \rangle $.

\par Hence, the identical dimensions of ${{\bf{R}}_k}$ and ${{\bf{\hat R}}_k}$ should increase in correlation, while the different dimensions should decrease. With this in mind, we develop a decomposition loss, denoted as ${l_{Dec}}$, that drives diagonal elements toward 1 and off-diagonal elements toward 0, formulated as follows:
\begin{equation} \label{eq15}
	\mathit{l}_{\mathit{Dec}} = \frac{1}{2}\left\| \mathbf{C} - \mathbf{I} \right\|_{\mathit{F}}^2,
\end{equation}
where ${\bf{I}}$ is Identity Matrix. By optimizing ${G_k}$ to minimize the decomposition loss ${l_{Dec}}$, we can successfully disentangle the raw and surrogate pair’s representations (i.e., $\bm{r}_{k,i}$ and $\hat{\bm{r}}_{k,i}$) into causal representations $\bm{r}^s_{k,i}$ and causal surrogate representations $\hat{\bm{r}}^s_{k,i}$.

\noindent \textbf{Task-revelence Augmentation.} Moreover, not all dimensions of causal representation $\bm{r}^s_{k,i}$ contribute equally to the primary task, some dimensions may carry insufficient task-relevant information and contribute minimally. Since Property 2 ensures dimensions are mutually independent with non-overlapping information, these suboptimal dimensions represent wasted capacity that could otherwise encode additional task-relevant knowledge. Therefore, we establish a CNN-based masker $\bm{\omega}$ that assigns a contribution score 
$z_j \in [0,1]$ to each dimension $j$ of $\bm{r}^s_{k,i}$:
\begin{equation} \label{eq17}
	{\bf{z}} = \text{Sigmoid}(\bm{\omega}(\bm{r}^s_{k,i})) \in \mathbb{R}^{1 \times J},
\end{equation}
and analogously $\hat{\mathbf{z}} = 
\mathrm{Sigmoid}(\bm{\omega}
(\hat{\bm{r}}^s_{k,i}))$ for the surrogate 
representation accordingly.

\par Then, we employ the derivable Gumbel-Softmax trick~\cite{Jang2017} to sample the score vector ${\bf{z}}$ to generate binary-like masks. The top-$\tau J$ highest-scoring dimensions are classified as \textit{upper dimensions} ${{\bm{m}}^{upp}}$, and the remaining ones as \textit{lower dimensions} $\bm{m}^{low} = \mathbf{1} - \bm{m}^{upp}$:
\begin{equation} \label{eq18}
	{\bm{m}^{upp}} = \mathrm{Gumbel-Softmax}(\bm{\omega}(\bm{r}^s_{k,i}), \tau J) \in \mathbb{R}^{1 \times J},
\end{equation}
and $\hat{\bm{m}}^{upp}$, $\hat{\bm{m}}^{low}$ 
are obtained analogously from 
$\hat{\bm{r}}^s_{k,i}$.

\par The masker $\bm{\omega}$ and generator $G_k$ are then trained adversarially. The masker is 
optimized to accurately distinguish upper from lower dimensions via score-based unsupervised 
losses:
\begin{equation} \label{eq21}
    l_{\mathrm{score}}^{\mathrm{upp}} = 
    \sum_{j=1}^J (1-z_j)^2 \cdot m_j^{upp} + 
    \sum_{j=1}^J (1-\hat{z}_j)^2 \cdot 
    \hat{m}_j^{upp},
\end{equation}

\begin{equation} \label{eq22}
    l_{\mathrm{score}}^{\mathrm{low}} = 
    \sum_{j=1}^J (z_j)^2 \cdot m_j^{low} + 
    \sum_{j=1}^J (\hat{z}_j)^2 \cdot 
    \hat{m}_j^{low}.
\end{equation}

\par The generator $G_k$ is trained by jointly minimizing $\mathcal{L}_{Dec}$ and maximizing the scores across all dimensions. This dual objective progressively enriches the information capacity of lower dimensions while preserving 
the task-relevant information in upper dimensions and respecting the independence constraint enforced by $\mathcal{L}_{Dec}$. The overall adversarial objective is:
\begin{equation} \label{eq23}
    \min_{G_k} \sum_{j=1}^J (1-z_j)^2 + 
    \lambda \mathcal{L}_{Dec}, \quad 
    \min_{\bm{\omega}} \mathcal{L}_{score}^{upp} 
    + \mathcal{L}_{score}^{low},
\end{equation}
where $\lambda$ denotes the trade-off factor. The passive party locally alternates between 
training the generator and masker for $M$ 
iterations, upon which the generator acquires 
the capability to extract clean and mutually 
independent causal representations. 
Algorithm~\ref{alg:causshield} in 
Appendix~\ref{append:alg} details \name's full 
procedure.

\section{Convergence Proof}
\label{sec:5}

\par We prove convergence (Theorem~\ref{theo:conver}) under the fixed step size strategy. More detailed proofs of Theorem~\ref{theo:conver} see Appendix~\ref{append:C}.
\par \textbf{Convergence analysis:} Prior assumptions are introduced that are commonly found in FL-related works~\cite{Li2020,Castiglia2022}.
\par The partial derivative for ${\theta _k}$ can be expressed as:
\begin{equation}\label{eq29}
	\begin{split}
		\nabla _k{\cal L}(\Theta)&=\frac{1}{N}\sum_{i=1}^N \nabla_{{\theta_k}}  {l_{ce}}(f_{\text{top}}(f_1(\theta_1;x_{1,i}),\\
		&f_2(\theta_2;x_{2,i}),\ldots,f_K(\theta_K;x_{K,i}),y_i)).
	\end{split}
\end{equation}
\par The set of samples and labels belonging to a mini-batch $B$ is denoted by ${{\bf{X}}^B}$ and ${{\bf{y}}^B}$ respectively. The stochastic partial derivative for the model parameter ${\theta _k}$ is described below:
\begin{equation}\label{eq30}
	\begin{split}
		\nabla _k{{\cal L}_B}(\Theta)&=\frac{1}{B}\sum_{i=1}^B \nabla_{{\theta_k}}  {l_{ce}}(f_{\rm top}(f_1(\theta_1;x_{1,i}),\\
		&f_2(\theta_2;x_{2,i}), \ldots,f_K(\theta_K;x_{K,i}),y_i)).
	\end{split}
\end{equation}
\par In addition, the chain rule shows that:
\begin{equation} \label{eq31}
	{\nabla _k}{{\cal L}_B}(\Theta ) = {\nabla _{{{\bf{R}}_k}}}{{\cal L}_B}(\Theta ) \cdot {\nabla _{{\theta _k}}}{{\bf{R}}_k}.
\end{equation}
\par Let $|| \cdot ||$ be the 2-norm and $|| \cdot |{|_{\cal F}}$ be the Frobenius norm of a matrix.

\begin{assumption}[Bounded hessian~\cite{Li2020}]\label{assump:bounded}
A positive constant ${M_k}$ exists for each passive party $k \in [1,K]$ such that the second partial derivative of ${\theta _k}$ satisfies:
\begin{equation} \label{eq32}
	||\nabla _{{{\bf{R}}_k}}^2{{\cal L}_B}(\Theta )||_{\cal F} \le {M_k}.
\end{equation}
\end{assumption}
\begin{assumption}[Bounded representation gradients~\cite{Castiglia2022}]\label{assump:bounded_grad}
There exist positive constants, denoted as ${\Phi _k}$, ensure that the cut layer gradient remains bounded:
\begin{equation} \label{eq33}
	||{\nabla _{{\theta _k}}}{{\bf{R}}_k}||_{\cal F} \le {\Phi _k}.
\end{equation}
\end{assumption}
\begin{definition}[Causal representation error]
Define ${\psi _{k,i}}$ be the causal representation error on the tensor ${x_{k,i}}$:
\begin{equation} \label{eq34}
	\text{${\psi _{k,i}} = {G_k}({\hat x_{k,i}}) - {f_k}({\theta _k};{x_{k,i}})$}.
\end{equation}
\end{definition}
\par Let ${\bm{\psi}}_k^t = [{({\psi _{k,1}})^T},{({\psi _{k,2}})^T},...,{({\psi _{k,B}})^T}]$ be the ${\mathbb{R}^{J \times B}}$ matrix in mini-batch $B$. We define the expected squared message error at communication round $t$ as $\Psi _k^t: = ||{\bm{\psi}}_k^t||_{\cal F}^2$.

\begin{theorem}\label{theo:bounded}
Under Assumptions~\ref{assump:bounded}-\ref{assump:bounded_grad}, the norm of the difference between the partial derivatives of causal and raw representations is bounded as:
\begin{equation} \label{eq35}
	\mathbb{E}||{\nabla _{_k}}{\cal L}_B^t({\bf{\hat R}}_k^s) - {\nabla _{_k}}{\cal L}_B^t({{\bf{R}}_k})||^2 \le M_k^2\Phi _k^2\sum\limits_{i = 0,i \ne k}^K {\Psi _i^t} .
\end{equation}
\end{theorem}

\par Further, some assumptions about local functions and stochastic gradient descent (SGD) algorithm are imported~\cite{Li2020}.
\begin{assumption}[Smoothness~\cite{Li2020}]\label{assump:smooth}
All local functions are L-Lipschitz smooth, such that for all $\bf{v}$ and $\bf{w}$, the objective function satisfies:
\begin{equation} \label{eq36}
	||\nabla {\cal L}({\bf{v}}) - \nabla {\cal L}({\bf{w}})|| \le L||{\bf{v}} - {\bf{w}}||,
\end{equation}
\begin{equation} \label{eq37}
	||{\nabla _k}{{\cal L}_B}({\bf{v}}) - {\nabla _k}{{\cal L}_B}({\bf{w}})|| \le {L_k}||{\bf{v}} - {\bf{w}}||,
\end{equation}
where ${\bf{v}} < \infty $ and ${\bf{w}} < \infty $.
\end{assumption}

\begin{assumption}[Unbiased gradients~\cite{Li2020}]\label{assump:unbiased}
 For $k \in [1,K]$ in batch $B$, the stochastic partial derivatives are unbiased:
\begin{equation} \label{eq38}
	\mathbb{E}{_B}[{\nabla _k}{{\cal L}_B}(\Theta )] = {\nabla _k}{\cal L}(\Theta ).
\end{equation}
\end{assumption}
\begin{assumption}[Bounded variance~\cite{Li2020}]\label{assump:varia}
The stochastic partial derivatives for ${\theta _k}$, are bounded as:
\begin{equation} \label{eq39}
	\mathbb{E}{_B}||{\nabla _k}{\cal L}(\Theta ) - {\nabla _k}{{\cal L}_B}(\Theta )||^2 \le \frac{{\pi _k^2}}{B}.
\end{equation}
where ${\pi _k} < \infty $.
\end{assumption}
\par Under Theorem~\ref{theo:bounded} and Assumptions~\ref{assump:bounded}-\ref{assump:smooth}, we prove convergence by limiting the boundary of the average squared gradient over $T$ communication rounds with a fixed step size. When the \name method is implemented, it is denoted as ${\cal L}(\hat \Theta )$, and when ${\cal L}(\Theta )$ represents the plain VFL.

\begin{theorem}[Convergence results]\label{theo:conver}
 If the learning rate $\eta $ is fixed in different communication rounds and satisfies $\eta  \le \frac{1}{{16T\max \{ L,{{\max }_k},{L_k}\} }}$, then the gradients of the \name algorithm is bounded over the global round $T$:
\begin{equation} \label{eq40}
	\begin{aligned}
		&\mathbb{E}{^T}\left[ {{{\left\| {\nabla {\cal L}({\Theta ^t})} \right\|}^2}} \right] \\
		&\le \frac{{4\left[ {{\cal L}({\Theta ^0}) - \left[ {{\cal L}({\Theta ^T})} \right]} \right]}}{{\eta T}}\quad + 6\eta TL\sum\limits_{k = 1}^K {\frac{{\pi _k^2}}{B}} \\
		&\quad + 92T\sum\limits_{k = 1}^K {\sum\limits_{t = 0}^{T - 1} {M_k^2\Phi _k^2} } \sum\limits_{j = 1,j \ne k}^K {\Psi _j^t} .
	\end{aligned}
\end{equation}
\end{theorem}

\begin{remark}
\par The first term in the inequality is based on the difference between the model initialization and the final trained model. The second term involves the variance of stochastic partial derivatives and Lipschitz constants. The third term is associated with the bounded causal representation error mentioned in Theorem~\ref{theo:bounded}.
\end{remark}
\begin{remark}
\par When $\eta  = \frac{1}{{\sqrt T }}$ and $\Psi  = \frac{1}{T}\sum\limits_{t = 0}^{T - 1} {\sum\limits_{k = 1}^K {\Psi _k^t} } $, the optimal convergence rate is obtained:
\begin{equation} \label{eq41}
	\mathbb{E}\left[ {||\nabla {\cal L}({\Theta ^T})|{|^2}} \right] = O\left(\frac{1}{{\sqrt{T}}} + \Psi \right).
\end{equation}
\end{remark}
\par We have now proved that the \name can withstand the bounded error and obtain stable asymptotic convergence when this condition is met.

\section{Experiments}
\label{sec:6}

\par We evaluate \name against seven SOTAs across defense effectiveness, model utility, and computational efficiency, demonstrating its superiority.

\subsection{Experimental Setup}
\par \textbf{Datasets and model settings:} Following the experimental settings of representative 
SOTAs~\cite{Xu2025, Yao2025, Erdogan2022}, we conduct our experiments primarily on five diverse datasets: MNIST, EMNIST, CIFAR-10, CIFAR-100, and ImageNet. The models for MNIST and EMNIST consist of 11 layers; those for CIFAR-10 and CIFAR-100 employ a ResNet-18. Notably, for the more complex ImageNet dataset, we utilize a deeper ResNet-34 model. We adopt the spatial splitting strategy described in \cite{Yao2025} to partition the feature dimensions held by each participant.
\par \textbf{Implementation and configuration:} Without loss of generality, we consider a configuration with one active party and two passive parties. We utilize SGD optimizer with a learning rate of 1e-4 and a momentum of 0.9. For all experiments, we set the number of surrogate iterations $M$ based on the complexity of the dataset. Concretely, we use $M = 20$ for MNIST and EMNIST, and $M = 30$ for CIFAR-10, CIFAR-100, and ImageNet. Moreover, as the model approaches convergence, it is feasible to gradually reduce the value of $M$ to lower computational cost and mitigate the risk of overfitting to synthetic surrogates. Additionally, we fix the masking threshold $\tau$ at 0.5, meaning that the top-50\% most causally informative dimensions are preserved during causal representation construction. 

\par \textbf{Compared methods:} We compare the performance and privacy of our \name against SOTAs. These comparison methods and their corresponding of privacy parameters are outlined as follows: (1) DP-L~\cite{Geyer2017} ($\varepsilon$: the privacy budget in DP based on Laplace noise); (2) Prune~\cite{Zhu2019} ($k$: the pruning rate of the representation elements); (3) Soteria~\cite{Sun2021} ($s$: the perturbation degree of the representation element that is most beneficial to the attacker); (4) NoPeek~\cite{Vepakomma2020} ($a$: the coefficient of information leakage reduction term added to the objective loss function); (5) DISCO~\cite{Singh2021} ($\rho$: the hyper-parameter to trade-off between accuracy and privacy); (6) ResSFL~\cite{Li2022c} ($\delta$: control the weight of the attacker-aware training term in the loss function); (7) InvL-DNP and InvL-ENP~\cite{Xu2025} ($n$: control the strength of the adaptively noise).


\par \textbf{Attack models:} Unsplit~\cite{Erdogan2022} minimizes the distance between the attack model and raw model by optimizing model parameters and shadow input through the coordinate gradient descent. URVFL~\cite{Yao2025} leverages a discriminator with an auxiliary classifier to generate malicious gradients that are indistinguishable from those in VFL training, while simultaneously utilizing labels to enhance reconstruction performance.
\par \textbf{Evaluation metrics:} To evaluate the defense effect of the \name, we employ peak signal-to-noise ratio (PSNR), mean squared error (MSE), structural similarity index measure (SSIM), learned perceptual image patch similarity (LPIPS) and attack success rate (ASR) as our metrics. To assess the impact of the defenses on the model utility of VFL, we show the accuracy (ACC) of the global model on the validation dataset over multiple communication rounds.

\subsection{Privacy Evaluations}

\noindent\textbf{Defense against Unsplit.} The defense effectiveness of \name on the MNIST and EMNIST datasets, as depicted in Fig.~\ref{Fig.6}. It is evident that without any defense mechanisms, Unsplit attackers can accurately replicate the contours of the digits. In contrast, the \name effectively prevents attackers from reconstructing handwritten digits and exhibits the highest MSE compared to other methods. Meanwhile, pruning-based approaches (i.e., Prune and Soteria) display unsatisfactory defense effectiveness, even when the pruning rate reaches 95\%. This finding indicates that even optimized pruning methods are vulnerable to exploitation when attackers possess certain prior knowledge. 

\begin{figure}[htbp]
	\centering
	\subfloat[MNIST]{\includegraphics[width=2.6in]{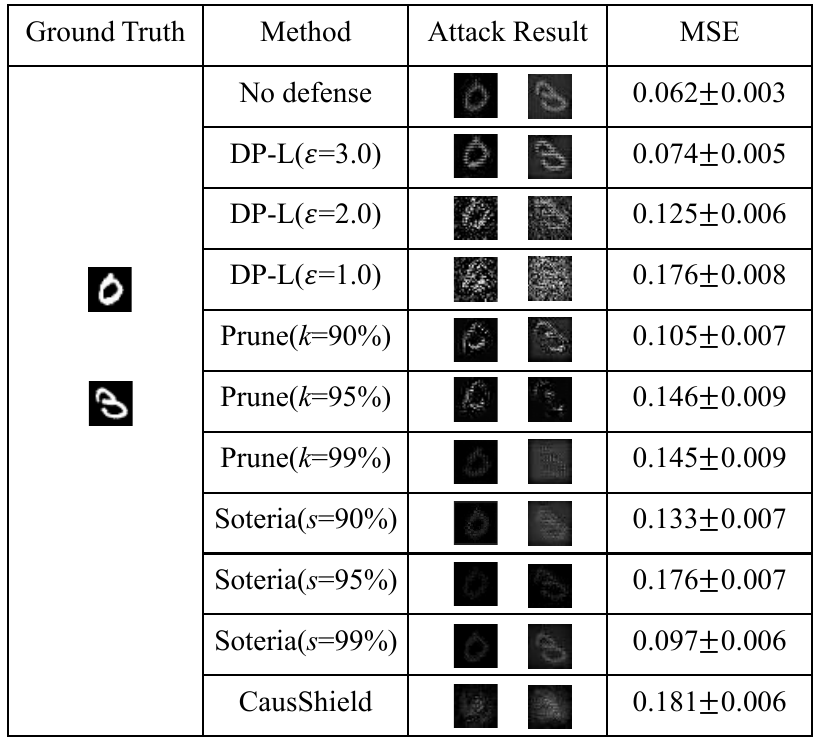}\label{Fig.6a}}
	\\
	\subfloat[EMNIST]{\includegraphics[width=2.7in]{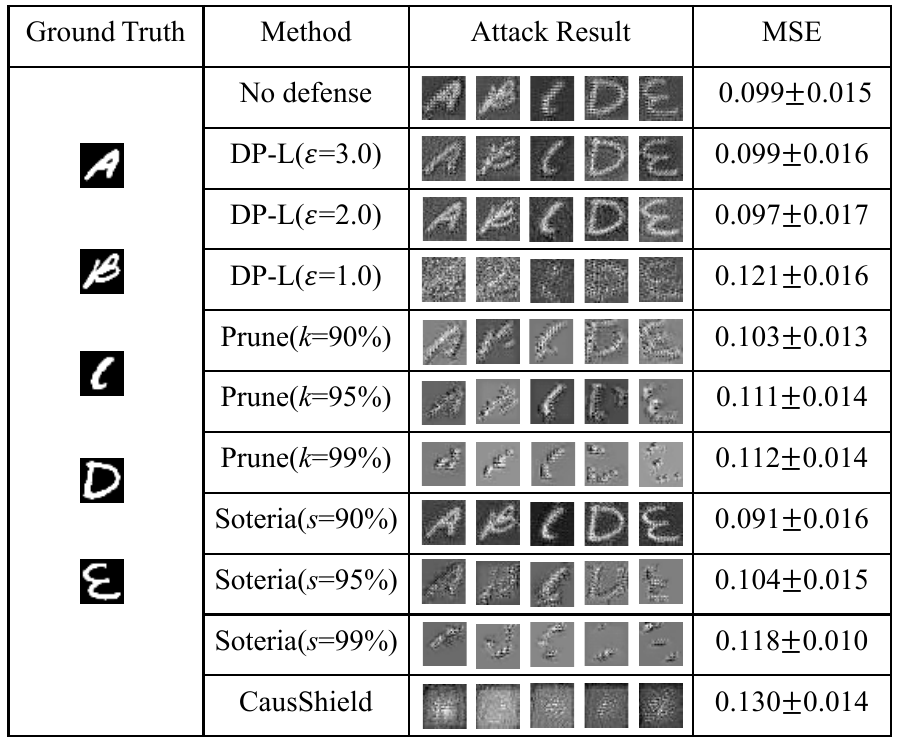}\label{Fig.6b}}
	\caption{Defense effects against Unsplit on MNIST/EMNIST.}
	\label{Fig.6}
\end{figure}

\begin{figure}[htb]
	\centering
	\includegraphics[width=2.4in]{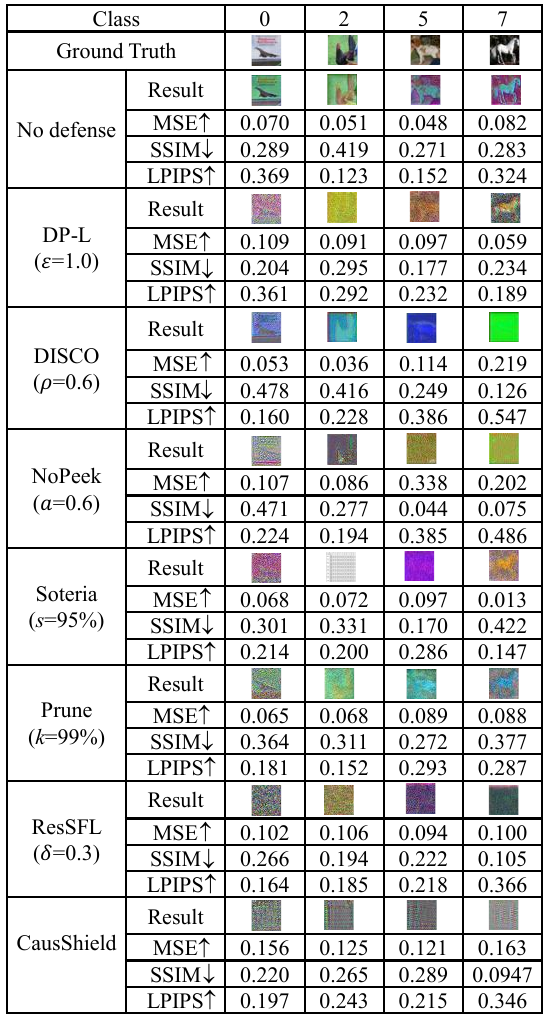}
	\caption{Defense effects against Unsplit on CIFAR10.}
	\label{Fig.7}
\end{figure}

\begin{figure}[htbp]
	\centering
	\includegraphics[width=2.6in]{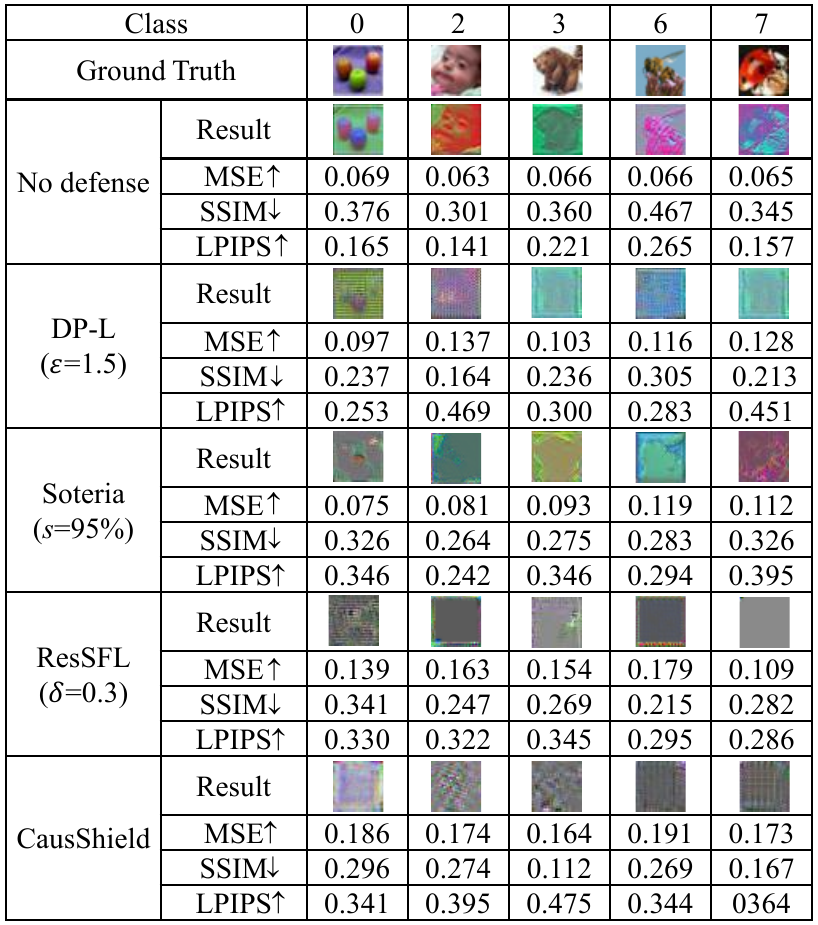}
	\caption{Defense effects against Unsplit on CIFAR-100.}
	\label{Fig.8}
\end{figure}

\begin{figure}[htb]
	\centering
	\includegraphics[width=2.6in]{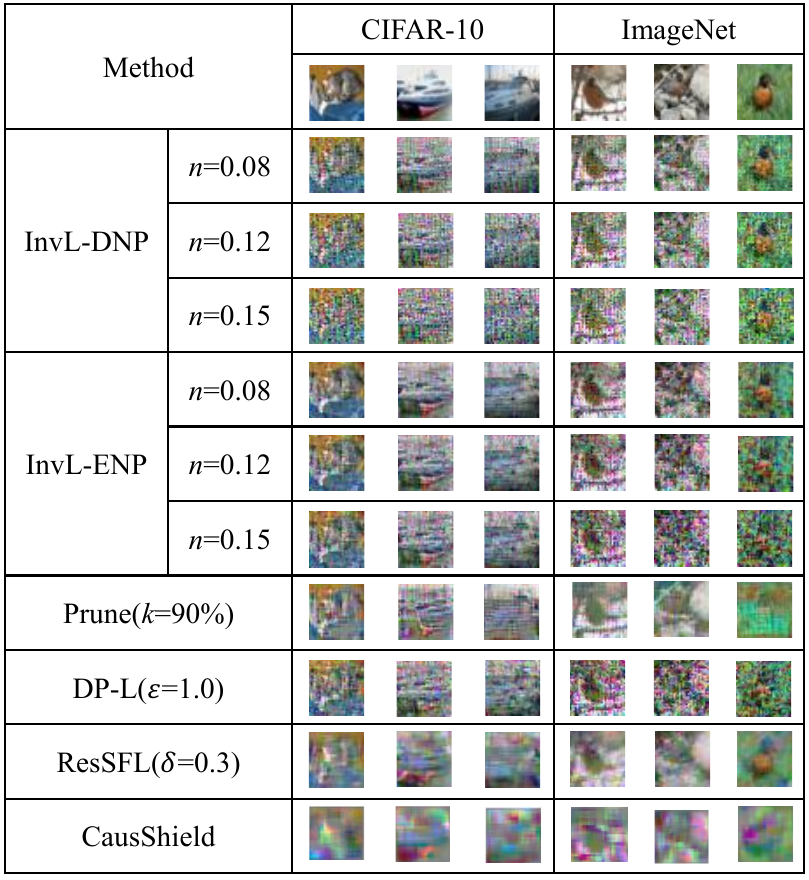}
	\caption{Defense effects against URVFL on CIFAR-10/ImageNet.}
	\label{Fig.imagenet}
\end{figure}

\begin{figure}[htb]
	\centering
	\includegraphics[width=2.8in]{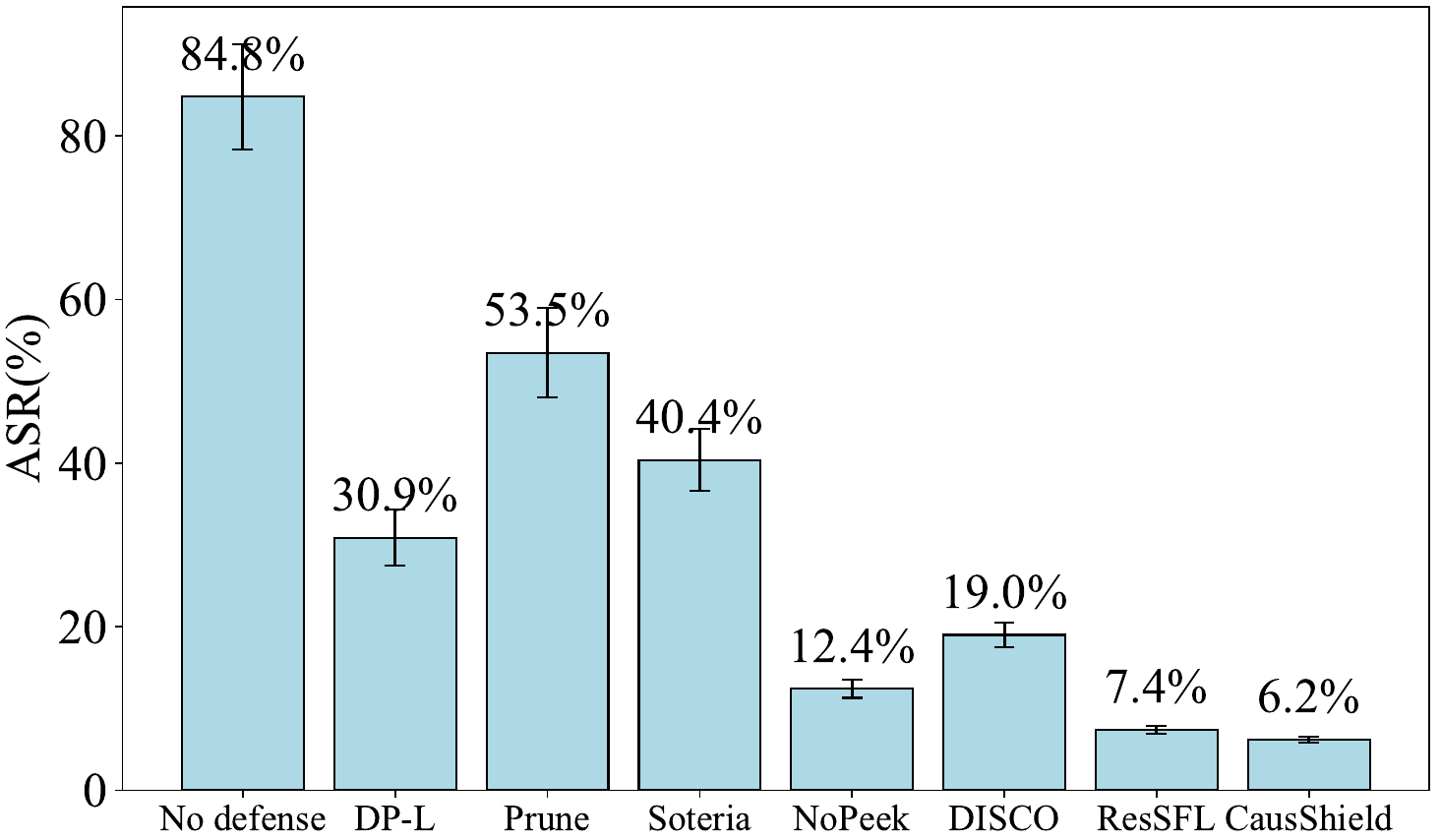}
	\caption{Comparison of the batch attack-defense against Unsplit.}
	\label{Fig.9}
\end{figure}

\begin{table*}[]
	\centering
	\caption{COMPARISON OF DEFENSE METHODS AGAINST URVFL~\cite{Yao2025} ATTACKS.}
	\label{table.new}
	\scalebox{0.85}{
	\begin{tabular}{c|cccc|cccc}
		\hline
		\multirow{2}{*}{\textbf{Method}} & \multicolumn{4}{c|}{\textbf{CIFAR-10}} & \multicolumn{4}{c}{\textbf{ImageNet}} \\ \cline{2-9} 
		& \multicolumn{1}{c|}{MSE$\,\uparrow$} & \multicolumn{1}{c|}{PSNR$\,\downarrow$} & \multicolumn{1}{c|}{SSIM$\,\downarrow$} & ACC$\,\uparrow$ & \multicolumn{1}{c|}{MSE$\,\uparrow$} & \multicolumn{1}{c|}{PSNR$\,\downarrow$} & \multicolumn{1}{c|}{SSIM$\,\downarrow$} & ACC$\,\uparrow$ \\ \hline
		InvL-DNP ($n=0.08$) & \multicolumn{1}{c|}{0.029±0.014} & \multicolumn{1}{c|}{15.925±2.036} & \multicolumn{1}{c|}{0.534±0.030} & 70.81\% & \multicolumn{1}{c|}{0.014±0.004} & \multicolumn{1}{c|}{18.559±1.228} & \multicolumn{1}{c|}{0.606±0.049} & 59.2\% \\ \hline
		InvL-DNP ($n=0.12$) & \multicolumn{1}{c|}{0.042±0.016} & \multicolumn{1}{c|}{14.133±1.674} & \multicolumn{1}{c|}{0.416±0.012} & 66.70\% & \multicolumn{1}{c|}{0.024±0.005} & \multicolumn{1}{c|}{16.349±0.888} & \multicolumn{1}{c|}{0.463±0.061} & 55.7\% \\ \hline
		InvL-DNP ($n=0.15$) & \multicolumn{1}{c|}{0.053±0.019} & \multicolumn{1}{c|}{13.057±1.640} & \multicolumn{1}{c|}{0.304±0.039} & 58.32\% & \multicolumn{1}{c|}{0.030±0.004} & \multicolumn{1}{c|}{15.231±0.646} & \multicolumn{1}{c|}{0.397±0.076} & 49.0\% \\ \hline
		InvL-ENP ($n=0.08$) & \multicolumn{1}{c|}{0.021±0.014} & \multicolumn{1}{c|}{17.775±2.885} & \multicolumn{1}{c|}{0.699±0.063} & 64.166\% & \multicolumn{1}{c|}{0.013±0.004} & \multicolumn{1}{c|}{19.013±1.374} & \multicolumn{1}{c|}{0.644±0.055} & 60.1\% \\ \hline
		InvL-ENP ($n=0.12$) & \multicolumn{1}{c|}{0.028±0.015} & \multicolumn{1}{c|}{16.255±2.405} & \multicolumn{1}{c|}{0.612±0.039} & 61.66\% & \multicolumn{1}{c|}{0.025±0.006} & \multicolumn{1}{c|}{16.177±1.159} & \multicolumn{1}{c|}{0.462±0.075} & 54.9\% \\ \hline
		InvL-ENP ($n=0.15$) & \multicolumn{1}{c|}{0.033±0.018} & \multicolumn{1}{c|}{15.423±2.319} & \multicolumn{1}{c|}{0.553±0.036} & 50.0\% & \multicolumn{1}{c|}{0.031±0.008} & \multicolumn{1}{c|}{15.248±1.163} & \multicolumn{1}{c|}{0.381±0.068} & 47.1\% \\ \hline
		\rowcolor{mixbg}\name & \multicolumn{1}{c|}{0.069±0.012} & \multicolumn{1}{c|}{11.114±1.411} & \multicolumn{1}{c|}{0.415±0.046} & 69.2\% & \multicolumn{1}{c|}{0.070±0.033} & \multicolumn{1}{c|}{12.622±2.432} & \multicolumn{1}{c|}{0.410±0.059} & 58.5\% \\ \hline
	\end{tabular}}
\end{table*}

\par Fig.~\ref{Fig.7} and Fig.~\ref{Fig.8} present the results of attack-defense experiments against Unsplit conducted on the CIFAR-10 and CIFAR-100 datasets, respectively. These evaluations include several advanced SOTAs. To account for potential color distortion in representation-based black-box Unsplit attacks, we incorporate SSIM and LPIPS as evaluation metrics alongside MSE. Notably, a defense against the Unsplit attacker is deemed successful when the MSE exceeds 0.1. 

\noindent\textbf{Defense against URVFL.} Table \ref{table.new} and Fig.~\ref{Fig.imagenet} illustrate the defensive performance of the InvL-DNP and InvL-ENP, alongside our \name, against advanced URVFL attackers. The evaluation spans various privacy parameter configurations on the CIFAR-10 and ImageNet datasets, where a defense is deemed successful if the reconstruction MSE exceeds 0.025. Additionally, in Figs.~\ref{Fig.7}-\ref{Fig.imagenet} and Table~\ref{table.new}, the cut layer in the model is the sixth layer.
 
\par Both the ResSFL and \name show superior defensive capability across all evaluation metrics compared to others. These two methods aim to minimize the inclusion of private information in transmitted representations. Privacy protection in ResSFL is attained through its attacker-aware training mechanism. Our \name, acting as a causal desensitization mechanism, discards task-irrelevant components and uploads only causal representations containing ample classification information. The baseline VFL framework without any defense achieves accuracies of 71.8\% and 63.3\% on CIFAR-10 and ImageNet, respectively. As shown in Table~\ref{table.new}, both InvL-DNP and InvL-ENP fail to provide adequate privacy protection when the adaptive noise scale is small (e.g., $n=0.08$). Although increasing the noise scale successfully defends URVFL attacks, it simultaneously imposes a significant penalty on the model's primary task utility. Notably, \name surpasses ResSFL in both defense effects (see Fig. \ref{Fig.7}-\ref{Fig.8}) and model utility (see Figs.~\ref{Fig.11}-\ref{Fig.12} and TABLE~\ref{tab:emnist-cifar100-iid}).

\noindent\textbf{Batch Attack Evaluation.} In turn, batch attack-defense experiments are conducted on the CIFAR-10 dataset, as depicted in Fig.~\ref{Fig.9}. The privacy parameter settings in batch attack-defense experiments follow those used in Fig.~\ref{Fig.7}. Batch experiments mitigate the influence of randomness and privacy parameters on evaluations, thereby yielding more accurate and comprehensive results. For each class in CIFAR-10, 1000 images are randomly selected to observe the ASR of the Unsplit. Fig.~\ref{Fig.9} illustrates the average ASR across 10 sets of 1000 images. Our \name achieves the lowest ASR. This suggests that the \name mechanism is effective against large-scale attacks and offers increased dependability and practicality.

\noindent\textbf{Robustness across Cut-Layer Depths.} Subsequently, we conduct extensive experiments on CIFAR-10 to determine whether these methods' privacy parameters required vary with changes in the cut-layer depth. As depicted in Fig.~\ref{Fig.10}, the quality of reconstructions gradually decreases as the cut-layer deepens, even without any defense. At a cut-layer depth of 12, the attacker cannot extract sample details from representations. This is attributed to the higher abstraction of high-level features in CNNs, which contain minimal low-level information such as edges, textures, and colors. Moreover, as the quality of reconstruction decreases, the required privacy parameters for methods such as DP-L, DISCO, and NoPeek decrease. Instead, our \name does not necessitate extensive experimentation for parameter tuning. This characteristic reduces time and resource consumption and aligns \name more closely with practical application. 


\begin{figure}[htb]
	\centering
	\includegraphics[width=2.5in]{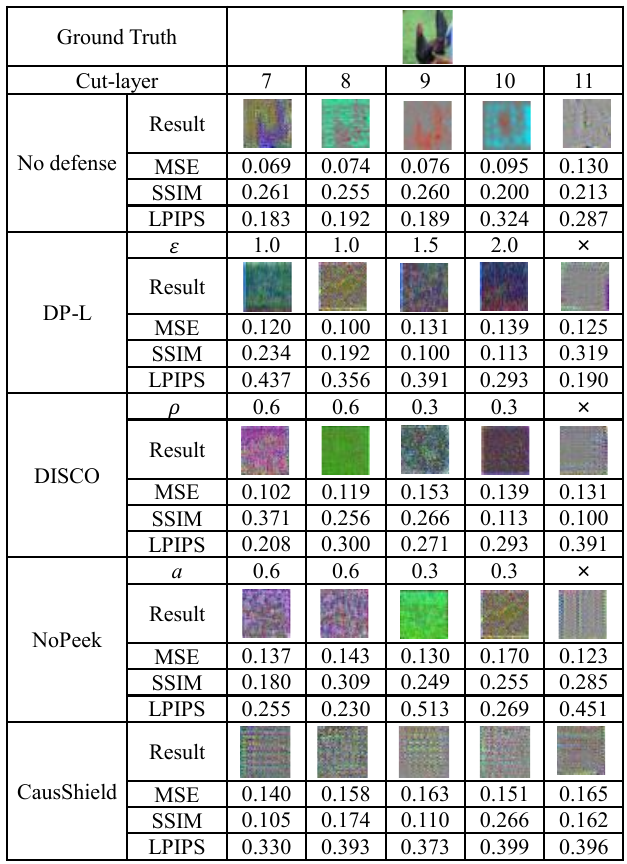}
	\caption{Defense effect on different cut-layers on CIFAR-10.}
	\label{Fig.10}
\end{figure}

\subsection{Model Utility Evaluations}


\noindent\textbf{Accuracy and Convergence.} We evaluate the model accuracy and loss across various approaches on the MNIST dataset (see Fig.~\ref{Fig.11}). The privacy parameters for each method are as follows: DP-L ($\varepsilon$=1.0), Prune ($k$=99\%), Soteria ($s$=95\%). It is evident that the accuracy and loss curves of the proposed \name closely align with those of vanilla VFL as the communication rounds increase, with negligible utility degradation. The accuracy of \name remains 2\% higher than that of the Soteria. Furthermore, the stability and convergence rate of our \name method are comparable to those of vanilla VFL. In contrast, other comparison approaches exhibit instability and slower convergence speeds.
\begin{figure}[htb]
	\centering
	\subfloat[ACC]{\includegraphics[width=1.7in]{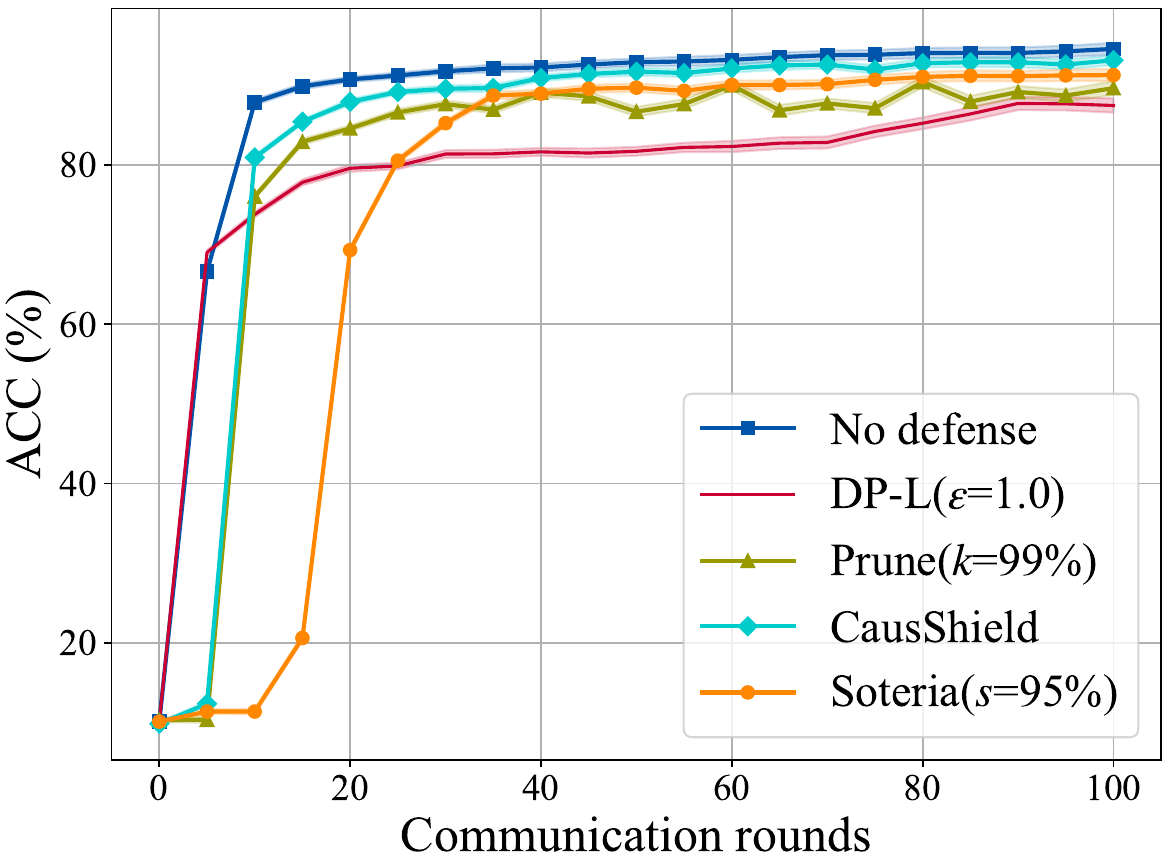}\label{Fig.11a}}
	\subfloat[Loss]{\includegraphics[width=1.7in]{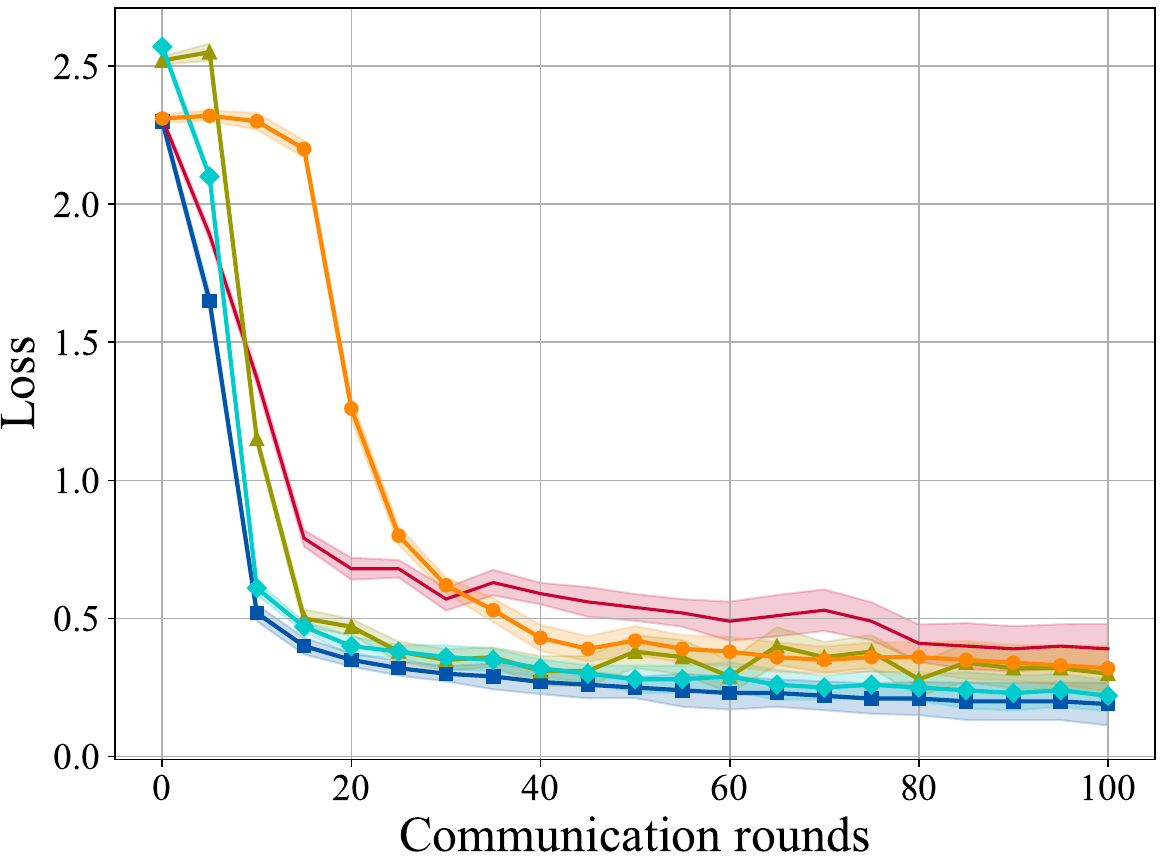}\label{Fig.11b}}
	\caption{Comparison of the ACC and Loss on MNIST.}
	\label{Fig.11}
\end{figure}
\begin{figure}[htb]
	\centering
	\subfloat[ACC]{\includegraphics[width=1.7in]{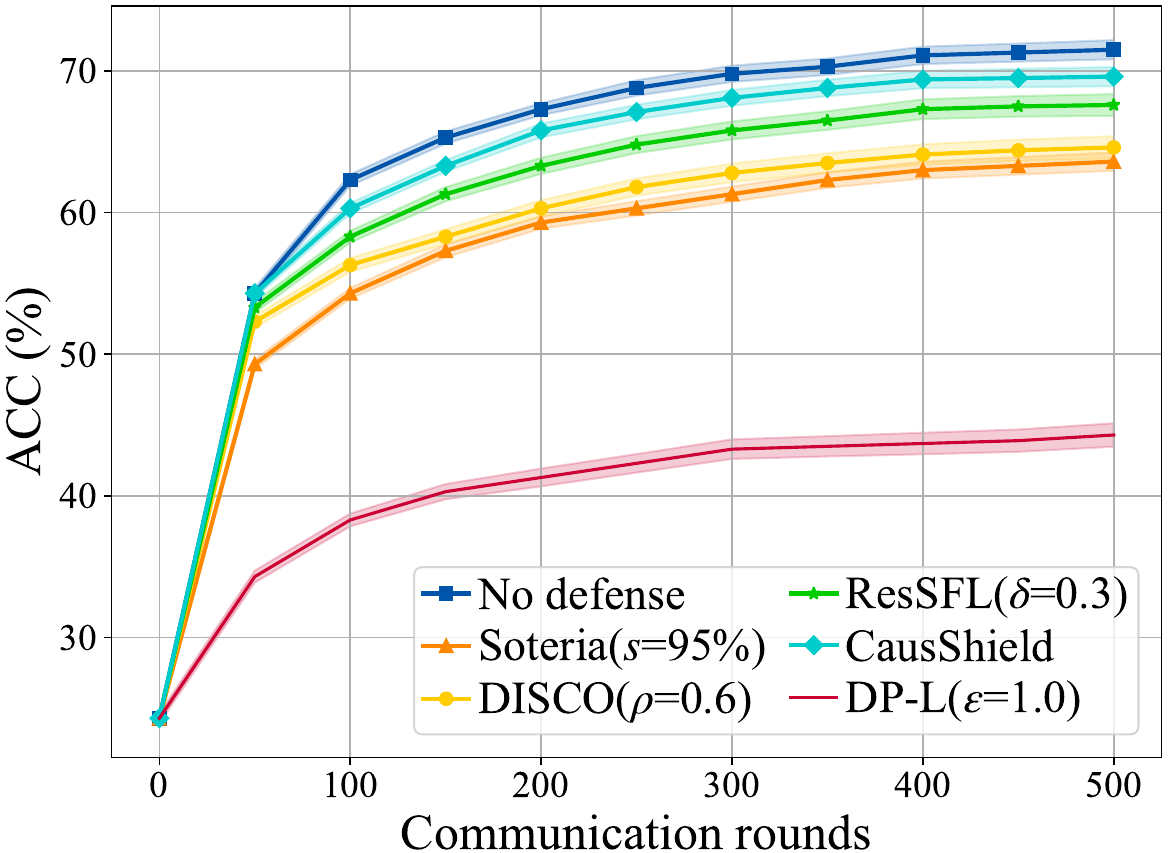}\label{Fig.12a}}
	\subfloat[Loss]{\includegraphics[width=1.7in]{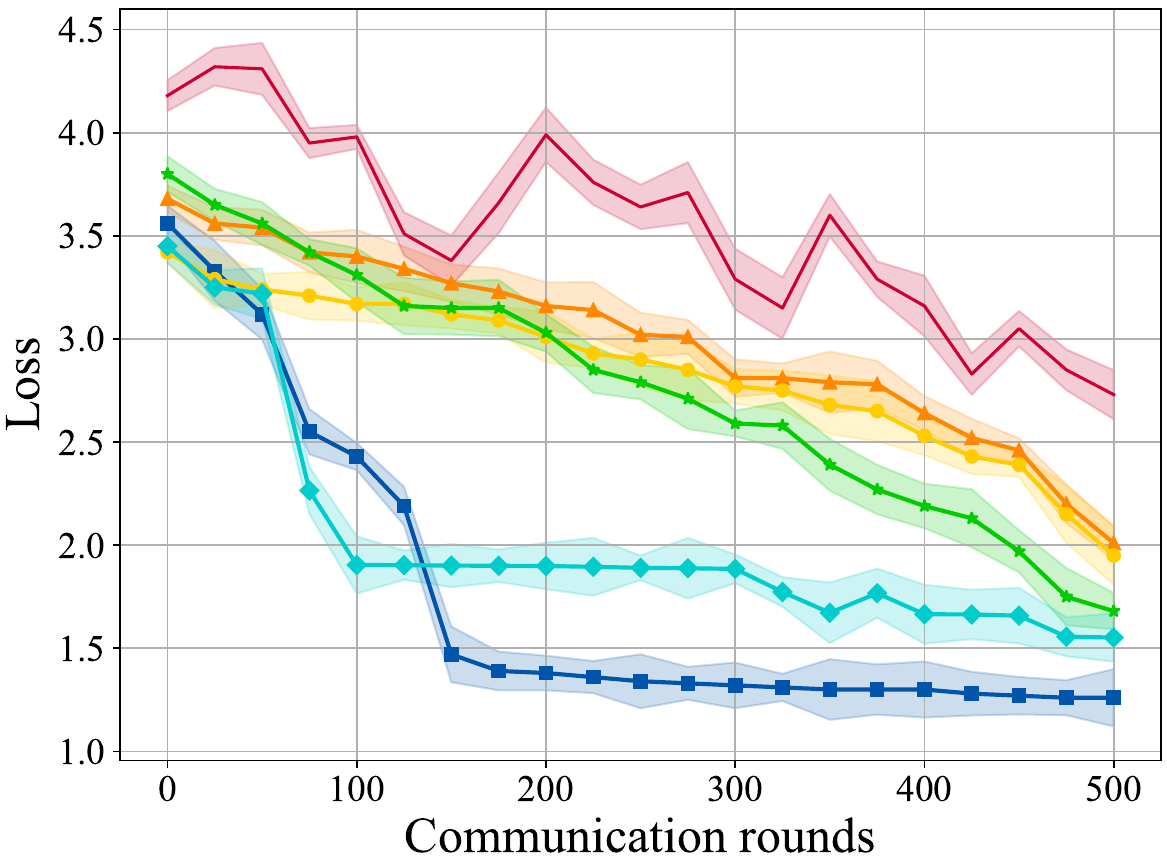}\label{Fig.12b}}
	\caption{Comparison of ACC and Loss on CIFAR-10.}
	\label{Fig.12}
\end{figure}
\begin{figure}[htb]
	\centering
	\subfloat[MNIST]{\includegraphics[width=1.7in]{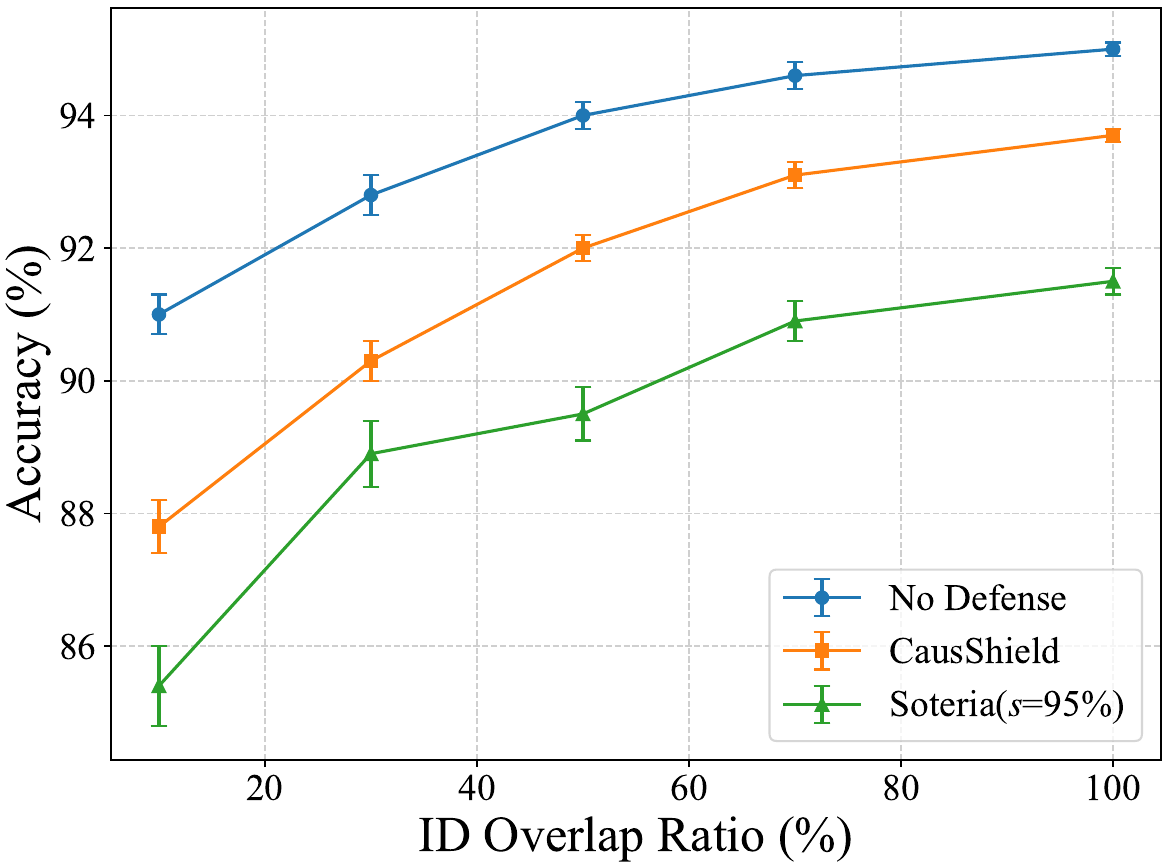}\label{Fig.13a}}
	\subfloat[CIFAR-10]{\includegraphics[width=1.7in]{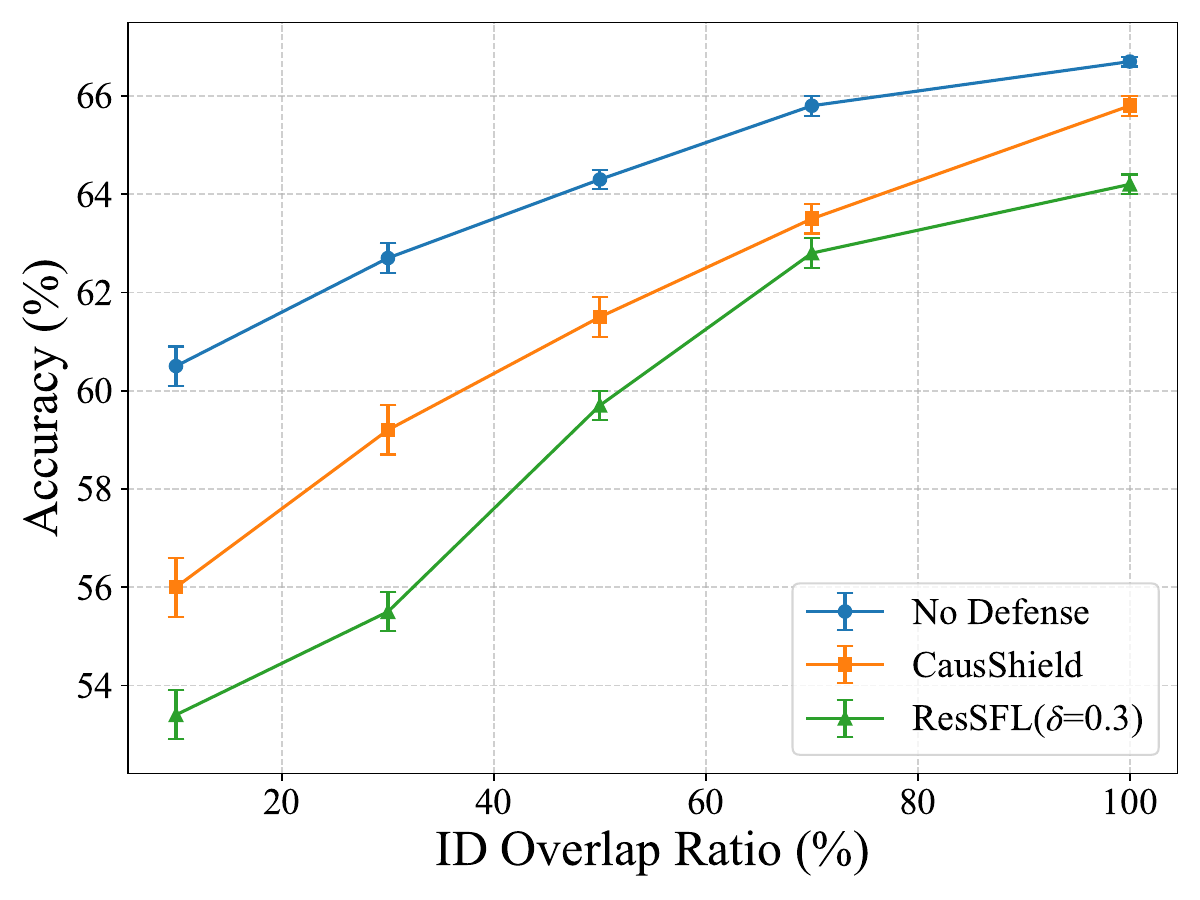}\label{Fig.13b}}
	\caption{Comparison of the ACC at varying ID overlap ratio.}
	\label{Fig.13}
\end{figure}
\begin{figure}[htb]
	\centering
	\includegraphics[width=2.35in]{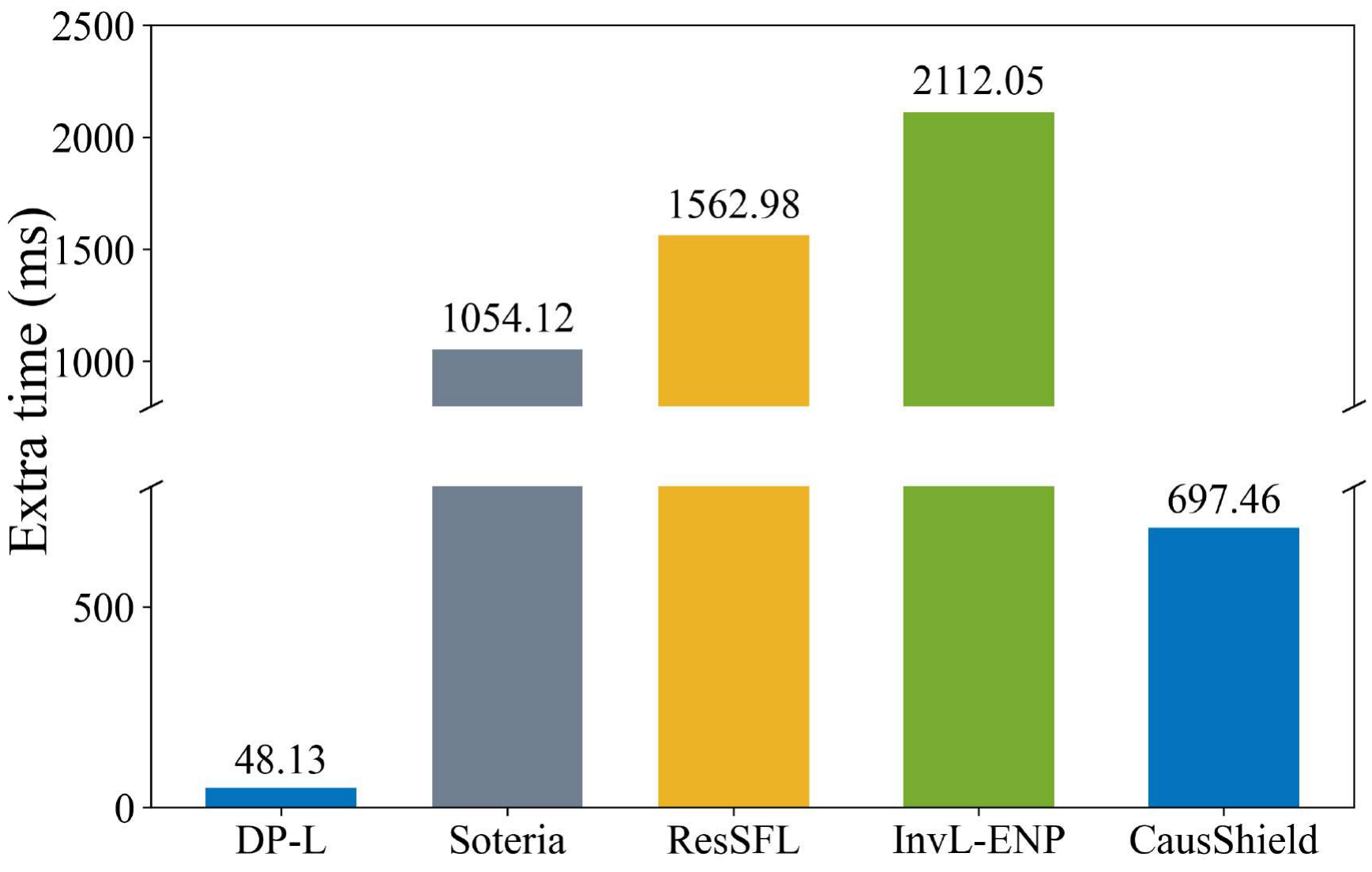}
	\caption{Illustration of extra processing time.}
	\label{Fig.14}
\end{figure}

\par On CIFAR-10, the privacy parameters for each method are set identically to those presented in Fig.~\ref{Fig.7}. Fig.~\ref{Fig.12} illustrates the round-wise testing accuracy and loss, where the following observations are made: (1) the unprotected VFL accuracy is the highest and serves as the baseline for utility evaluation. However, as mentioned earlier, it fails to defend against reconstruction attacks. (2) The noise perturbation based on the DP mechanism results in a significant decrease in model accuracy, approximately 15\%. This decrease is attributed to the irreversible disruption of intermediate representations caused by indiscriminate noise addition. (3) \name significantly outperforms SOTAs, with the utility degradation being almost negligible. This phenomenon demonstrates that \name reduces utility loss by largely preserving crucial information for classification during causal representation learning. Thus, \name validates our design objective of addressing the privacy-utility trade-off more flexibly. As also shown in TABLE~\ref{tab:emnist-cifar100-iid}, \name consistently achieves superior ACC on the extended EMNIST and CIFAR-100 datasets. These results align with the findings on MNIST and CIFAR-10, demonstrating the scalability of our method.

\begin{table}[htbp]
    \centering
    \caption{COMPARISON OF DEFENSE METHODS IN ACC(\%) 
    ON EMNIST and CIFAR-100.}
    \begin{minipage}{0.45\linewidth}
        \centering
        \textbf{(a) EMNIST}
        \scalebox{0.85}{\begin{tabular}{lc}
            \toprule
            \textbf{Method} & \textbf{ACC} \\
            \midrule
            No defense & 89.21 $\pm$ 0.85 \\
            DP-L($\varepsilon$=1.0) & 82.13 $\pm$ 1.10 \\
            Prune($k$=99\%) & 84.72 $\pm$ 0.92 \\
            Soteria($s$=95\%) & 86.58 $\pm$ 0.88 \\
            \rowcolor{mixbg}
            \name & 88.32 $\pm$ 0.75 \\
            \bottomrule
        \end{tabular}}
    \end{minipage}
    \hfill
    \begin{minipage}{0.45\linewidth}
        \centering
        \textbf{(b) CIFAR-100}
        \scalebox{0.8}{\begin{tabular}{lc}
            \toprule
            \textbf{Method} & \textbf{ACC} \\
            \midrule
            No defense & 62.20 $\pm$ 0.65 \\
            DP-L($\varepsilon$=1.5) & 38.80 $\pm$ 1.20 \\
            Soteria($s$=95\%) & 53.40 $\pm$ 2.00 \\
            ResSFL($\delta$=0.3) & 55.60 $\pm$ 1.40 \\
            \rowcolor{mixbg}
            \name & 56.90 $\pm$ 1.30 \\
            \bottomrule
        \end{tabular}}
    \end{minipage}
    \label{tab:emnist-cifar100-iid}
\end{table}

\noindent\textbf{Robustness under Asymmetric VFL.} To evaluate the robustness of our \name under realistic vertical heterogeneity, we extend our experiments to the asymmetrical VFL setting by varying the sample identifier (ID) overlap ratio between clients (from 10\% to 100\%), following the AVFL framework~\cite{Yang2020}. Note that our previous experiments were conducted under the symmetric VFL assumption with 100\% ID overlap. The asymmetrical setting reflects practical scenarios where smaller participants may only hold a partial subset of the global sample IDs. The privacy parameters for each method are as follows: Soteria ($s$=95\%), ResSFL ($\delta$=0.3). As shown in Fig.~\ref{Fig.13}, all methods suffer from performance degradation as the ID overlap ratio decreases. Nevertheless, our proposed \name consistently outperforms baseline methods (e.g., Soteria, ResSFL) across all overlap levels, demonstrating strong resilience to ID space asymmetry.

\subsection{Computational Efficiency Evaluations}

Fig.~\ref{Fig.14} illustrates the extra processing time required for different defense methods within a single communication round ($t$=1) on CIFAR-10 dataset. The privacy parameters used are aligned with those used in Fig.~\ref{Fig.7}, and we set $n=0.12$ for InvL-ENP. As a typical lightweight approach, DP-L only adds noise to representations and thus introduces minimal latency. The per-round overhead of \name arises from the causal representation learning module. The surrogate dataset generation, being a one-time offline operation, contributes no per-round cost. Furthermore, as VFL training progresses and the bottom model $f_k$ stabilizes, the generator initialized from $f_k$ requires fewer iterations $M$ to converge within each round, allowing $M$ to be progressively reduced and the per-round overhead to decrease accordingly.


\par In contrast, both the Soteria, InvL-ENP, and ResSFL methods noticeably slow down the training speed. Soteria requires a comprehensive traversal of each element within the representations, calculating their corresponding privacy benefits, and then conducting sorting and pruning operations. The primary computational overhead of InvL-ENP stems from the calculation of the Jacobian matrix for each input instance and the subsequent singular value analysis required to derive the adaptive noise scale. Similarly, the ResSFL method involves locally simulating attackers, iteratively adjusting the parameters of the base model and attack model until the reconstructed images no longer leak privacy information. In summary, our approach provides robust privacy protection and offers substantial benefits in terms of computational efficiency and sustained model performance.

The proposed \name introduces negligible communication overhead compared to the standard VFL baseline. It is evidenced by two key observations: first, the convergence behavior, as shown in Fig.~\ref{Fig.11}(b) and Fig.~\ref{Fig.12}(b), is nearly identical, confirming no added communication rounds; second, the uploaded causal representations maintain the same dimensionality as their original counterparts, guaranteeing an unchanged per-round transmission volume. Therefore, the advancement in data privacy afforded by \name comes at no extra communication cost.

\section{Conclusions}
\label{sec:7}

\par This work proposes \name to defend against sample reconstruction attacks in VFL. Grounded in the SCM perspective and our established theoretical foundation, \name decomposes intermediate representations and transmits only causal representations to eliminate the primary reconstruction attack surface while preserving model utility. This decomposition is achieved through two unsupervised modules: surrogate dataset generation and causal representation learning, requiring neither ground-truth labels nor task-specific supervision, ensuring seamless integration into standard VFL pipelines. Theoretical analysis confirms that transmitting only causal representations can resist sample reconstruction, and \name retains the convergence behavior of standard VFL. Extensive experiments demonstrate that \name consistently outperforms seven SOTAs in privacy protection, model utility, and computational efficiency.



\appendix

\subsection{\name's Algorithm~\ref{prop:singular}}
\label{append:alg}

\par The algorithm of \name is presented in Algorithm~\ref{alg:causshield}

\begin{algorithm}[htbp]
	\renewcommand{\algorithmicrequire}{\textbf{Input:}}
	\renewcommand{\algorithmicensure}{\textbf{Output:}}
	\renewcommand{\algorithmicreturn}{\textbf{Begin}}
	\caption{The \name algorithm}\label{alg:causshield}
	\begin{algorithmic}[1]
		\vspace{.2cm}
		\REQUIRE ~\\
		Variable value: $K$; $N$; $D$; $\hat D$; $B$; $\lambda$; $T$; Bottom models and its parameters $\{ {f_k};{\theta _k}\} _{k = 1}^K$; Top model and it's parameters $\{ {f_{\rm top}};{\theta _{\rm top}}\} $; Number of causal module optimizations $M$.\\
		\ENSURE ~\\
		The optimal model parameters $\{ {\theta _k}^*\} _{k = 1}^K$ and $\theta _{\rm top}^*$.\\
		\STATE \textbf {Active party executes:}\\
		\FOR {each communication round $t \in [1,T]$}
		\STATE Concatenates received representations ${{\bf{R}}^s} = [{\bf{R}}_1^s,{\bf{R}}_2^s,...,{\bf{R}}_K^s]$;
		\STATE Calculates the CE loss and obtain gradients $\nabla {\theta _{\rm top}}$;
		\STATE Updates the head model ${f_{\rm top}}$;
		\STATE Sends the cut-layer gradients ${\nabla _{{\theta _{\rm cut}}}}{\cal L}$ to passive parties;
		\ENDFOR
		\STATE \textbf {Passive parties $k \in [1,K]$ executes:}\\
		\FOR {each batch}
		\STATE Conducts forward propagation to compute intermediate representations ${{\bf{R}}_k} = \{ {f_k}({\theta _k};{x_{k,i}})\} _{i = 1}^B$;
		\FOR {each causal optimization module}
		\STATE Feeds raw data ${x_{k,i}}$ and corresponding surrogate data ${\hat x_{k,i}}$ to the generator ${G_k}$;
		\STATE Perform Z-score normalization on representations ${{\bf{R}}_k}$ and ${{\bf{\hat R}}_k}$;
		\STATE Computes decomposition loss defined in Eq. (\ref{eq15});
		\STATE Calculates the causal representations ${\bf{R}}_k^s$ and ${\bf{\hat R}}_k^s$;
		\STATE Establishes a Masker $\omega $ to evaluate each dimension contribution within the representations ${\bf{R}}_k^s$ and ${\bf{\hat R}}_k^s$ as Eq. (\ref{eq17});
		\STATE Obtain derived masks ${{\bm{m}}_{upp}}$ and ${{\bm{m}}_{low}}$ as Eq. (\ref{eq18});
		\STATE Computes two unsupervised losses $l_{score}^{upp}$ and $l_{score}^{low}$ as Eqs. (\ref{eq21})-(\ref{eq22});
		\STATE Optimizes the ${G_k}$, $\bm{\omega}$ according to Formula (\ref{eq23});
		\ENDFOR
		\STATE Updates the causal representations ${\bf{R}}_k^s$;
		\IF {receives the gradients from the active party}
		\STATE {Updates local bottom model ${f_k}$ by passing the cut-layer gradients.}
		\ENDIF
		\ENDFOR
	\end{algorithmic}
\end{algorithm}

\subsection{Proof of Theorem~\ref{prop:singular}}
\label{append:A}

\begin{proof}
The Jacobians $\partial U/\partial x$ and $\partial S/\partial x$ 
satisfy the following structural properties of the SCM:
\begin{equation}
    \sigma_k\!\left(\frac{\partial U}{\partial x}\right) \geq \mu > 0, 
    \quad 
    \left\|\frac{\partial S}{\partial x}\right\|_2 \leq \vartheta < \infty,
\end{equation}
where $\mu$ and $\vartheta$ are of comparable order; the first ensures 
the mapping from $x$ to $U$ is non-degenerate across all top-$k$ 
directions, and the second ensures $S$ has bounded sensitivity 
to $x$. Since $\mathrm{row}(G^S_x) \perp \mathrm{row}(G^U_x)$ 
by Definition~\ref{def:jacobian}, $G_x$ decomposes along 
orthogonal subspaces $\mathcal{V}^S$ and $\mathcal{V}^U$.

\textbf{Step 1: Lower bound on $\sigma_i(G^U_x)$.}
By the chain rule, $G^U_x = \frac{\partial \bm{r}^u}{\partial U} 
\cdot \frac{\partial U}{\partial x}$.
Since $\sigma_{\min}(\partial U/\partial x) \geq 
\sigma_k(\partial U/\partial x) \geq \mu$, applying the 
singular value product lower bound 
$\sigma_i(AB) \geq \sigma_i(A)\cdot\sigma_{\min}(B)$,
and noting that $i \leq k$ implies 
$\sigma_i(\partial \bm{r}^u/\partial U) \geq 
\sigma_k(\partial \bm{r}^u/\partial U) \geq \epsilon$ 
by Assumption~\ref{assump:noncausal}:
\begin{equation}
    \sigma_i(G^U_x) 
    \geq \sigma_i\!\left(\frac{\partial \bm{r}^u}{\partial U}\right)
    \cdot \sigma_{\min}\!\left(\frac{\partial U}{\partial x}\right)
    \geq \epsilon \cdot \mu, 
    \quad \forall\, i \in \{1,\ldots,k\}.
\end{equation}

\textbf{Step 2: Upper bound on $\sigma_i(G^S_x)$.}
By the chain rule, $G^S_x = \frac{\partial \bm{r}^s}{\partial S} 
\cdot \frac{\partial S}{\partial x}$.
Applying submultiplicativity 
$\sigma_i(AB) \leq \sigma_1(A)\cdot\sigma_1(B)$, 
together with Assumption~\ref{assump:causal}:
\begin{equation}
    \sigma_i(G^S_x)
    \leq \left\|\frac{\partial \bm{r}^s}{\partial S}\right\|_2
    \cdot \left\|\frac{\partial S}{\partial x}\right\|_2
    \leq \gamma \cdot \vartheta,
    \quad \forall\, i \in \{1,\ldots,k\}.
\end{equation}

\textbf{Step 3: Spectral dominance.}
Since $\gamma \ll \epsilon$ by Assumption~\ref{assump:causal} 
and $\mu, \vartheta$ are of comparable order, the spectral gap 
$\delta := \epsilon\cdot\mu - \gamma\cdot \vartheta \gg 0$, yielding:
\begin{equation}
    \sigma_i(G^U_x) \geq \epsilon\cdot\mu 
    \gg \gamma\cdot \vartheta \geq \sigma_i(G^S_x), 
    \quad \forall\, i \in \{1,\ldots,k\}.
\end{equation}

\textbf{Step 4: Dominant alignment with $\mathcal{V}^U$.}
Since $\mathrm{row}(G^S_x) \perp \mathrm{row}(G^U_x)$ and 
the spectral gap $\delta \gg 0$, applying the Wedin 
theorem to bound the perturbation of the 
top-$k$ right singular vectors of $G_x$ away from 
$\mathcal{V}^U$ gives:
\begin{equation}
    \left\|P_{\mathcal{V}^S} V_{J,i}\right\| 
    \leq \frac{\gamma \cdot \vartheta}{\delta} \ll 1,
    \quad \forall\, i \in \{1,\ldots,k\},
\end{equation}
where $P_{\mathcal{V}^S}$ is the orthogonal projector onto 
$\mathcal{V}^S$. Since $\|P_{\mathcal{V}^U} V_{J,i}\|^2 + 
\|P_{\mathcal{V}^S} V_{J,i}\|^2 = 1$, it follows that:
\begin{equation}
    \sum_{i=1}^{k} \left\|P_{\mathcal{V}^U} V_{J,i}\right\|^2 
    \gg 
    \sum_{i=1}^{k} \left\|P_{\mathcal{V}^S} V_{J,i}\right\|^2,
\end{equation}
where $P_{\mathcal{V}^U}$ is the orthogonal projector onto 
$\mathcal{V}^U$. Therefore, a rank-$k$ optimal 
attacker~\cite{Xu2025} primarily exploits $\mathcal{V}^U$ 
for reconstruction, confirming that non-causal components 
constitute the dominant reconstruction attack surface.
\end{proof}

\subsection{Proof of Theorem~\ref{thm:invloss}}
\label{append:B}

\begin{proof}
\textbf{Step 1: Defense success threshold $\tau_{\mathrm{th}}$.} We select the ENP-based defense's InvLoss upper bound for a rank-$k$ optimal 
attacker in~\cite{Xu2025} as the $\tau_{\mathrm{th}}$:
\begin{equation}
    \tau_{\mathrm{th}} := \sum_{i=k+1}^{d}(V_{J,i}^\top x)^2 
    + \sum_{i=1}^{k}
    \frac{(U_i^\top \varepsilon)^2}{\sigma_i^2 \cdot p} 
    + C_0,
\end{equation}
where the first term is the baseline tail residual 
energy without any defense, and the second term 
reflects the additional InvLoss increase contributed 
by the injected noise $\varepsilon$.

\textbf{Step 2: Reconstruction error lower bound for 
$\bm{r}^s$.}
When only $\bm{r}^s$ is transmitted, the attacker's 
Jacobian reduces to $G^S_x$. By 
Assumption~\ref{assump:causal}, $\sigma_i(G^S_x) \leq 
\gamma$ for all $i \in \{1,\ldots,d\}$. The maximum 
energy recoverable from $x^s$ via the rank-$k$ 
approximation is:
\begin{equation}
    \sum_{i=1}^{k}(V^{S\top}_{J,i} x^s)^2 
    \leq k \cdot \sigma_1^2(G^S_x) \cdot \|x^s\|^2
    \leq k\gamma^2\|x^s\|^2.
\end{equation}
Since $\{V^S_{J,i}\}_{i=1}^{d}$ form an orthonormal 
basis:
\begin{equation}
    \sum_{i=k+1}^{d}(V^{S\top}_{J,i} x^s)^2 
    = \|x^s\|^2 - \sum_{i=1}^{k}
    (V^{S\top}_{J,i} x^s)^2
    \geq (1 - k\gamma^2)\|x^s\|^2,
\end{equation}
yielding the lower bound $\mathrm{InvLoss}(\bm{r}^s) \geq  (1 - k\gamma^2)\|x^s\|^2 + C_0$.

\textbf{Step 3: $\mathrm{InvLoss}(\bm{r}^s)$ strictly 
exceeds $\tau_{\mathrm{th}}$.}
Comparing Steps~1-2, since $\gamma \ll \epsilon$:
\begin{equation}
    \mathrm{InvLoss}(\bm{r}^s) \geq 
    (1-k\gamma^2)\|x^s\|^2 + C_0 
    \approx \|x^s\|^2 + C_0.
\end{equation}
For $\tau_{\mathrm{th}}$, the tail residual term satisfies:
\begin{equation}
    \sum_{i=k+1}^{d}(V_{J,i}^\top x)^2 
    \approx \sum_{i=k+1}^{d}(V_{J,i}^\top x^U)^2,
\end{equation}
since by Theorem~\ref{prop:singular} the leading $k$ 
singular vectors are predominantly aligned with 
$\mathcal{V}^U$, leaving only a small residual in 
the tail. The noise term 
$\sum_{i=1}^{k}(U_i^\top\varepsilon)^2/(\sigma_i^2 
\cdot p)$ is also finite for any practical noise 
level $\delta$. Therefore:
\begin{equation}
\begin{aligned}
    \mathrm{InvLoss}(\bm{r}^s) 
    &\approx \|x^s\|^2 + C_0 \\
    &\gg \sum_{i=k+1}^{d}(V_{J,i}^\top x^U)^2 
    + \sum_{i=1}^{k}
    \frac{(U_i^\top\varepsilon)^2}{\sigma_i^2 \cdot p} 
    + C_0 
    \geq \tau_{\mathrm{th}},
\end{aligned}
\end{equation}
confirming $\mathrm{InvLoss}(\bm{r}^s) \gg \tau_{\mathrm{th}}$, 
i.e., transmitting only causal representations 
provably exceeds the ENP defense threshold, rendering 
sample reconstruction infeasible regardless of the 
attack strategy.

\textbf{Step 4: Comparison with full representation.}
For the full representation $\bm{r}^s + \bm{r}^u$, 
by Theorem~\ref{prop:singular} the top-$k$ singular 
directions are dominated by $\mathcal{V}^U$, so the 
rank-$k$ approximation recovers substantial energy 
from $x^U$, yielding a much smaller tail residual 
and lower InvLoss:
\begin{equation}
\begin{aligned}
    \mathrm{InvLoss}(\bm{r}^s) 
    &\geq (1-k\gamma^2)\|x^s\|^2 + C_0 \\
    &\gg \sum_{i=k+1}^{d}(V_{J,i}^\top x)^2 + C_0 \geq \mathrm{InvLoss}(\bm{r}^s + \bm{r}^u),
\end{aligned}
\end{equation}
where the middle $\gg$ follows from 
$\sigma_i(G^U_x) \gg \sigma_i(G^S_x)$ established 
in Theorem~\ref{prop:singular}. Therefore, we achieve $\mathrm{InvLoss}(\bm{r}^s) \gg 
    \mathrm{InvLoss}(\bm{r}^s + \bm{r}^u)$.
\end{proof}

\subsection{Proof of Theorems~\ref{theo:bounded}-\ref{theo:conver}}
\label{append:C}

\renewcommand\thelemma{\arabic{lemma}}
\begin{lemma}\label{lemma1}
	The squared norm of the partial derivatives of embeddings belongs to party $k$ multiplied by the Taylor series terms $R_{top}^t({\bf{R}_k} + {\bf{E}}_k^t)$ is bounded~\cite{Castiglia2022}:
	\begin{equation} \label{eq35}
		||{\nabla _{{\theta _k}}}{\bf{R}_k} \cdot R_{top}^t({\bf{R}_k} + {\bf{E}}_k^t)||^2 \le M_k^2\Phi _k^2||{\bf{E}}_k^t||_{\cal F}.
	\end{equation}
	where $R_{top}^t$ is denotes as the infinite sum of all terms in this Taylor series from the second partial derivatives; we let ${\bf{E}}_k^t$ be the causal representation error on each representation used in gradient calculation at iteration $t$.
\end{lemma}
\begin{lemma}[Bounded variance]\label{lemma2}
	Under \textbf{Assumptions~\ref{assump:bounded}-\ref{assump:varia}}, if $\eta  \le \frac{1}{{4T{{\max }_k}{L_k}}}$, we can bound the expected squared norm difference between $\nabla {{\cal L}_B}({\Theta ^{{t_0}}})$ at the iteration ${t_0}$ and $\nabla {{\cal L}_B}({\hat \Theta ^{{t_c}}})$ at the current iteration ${t_c}$ is expressed as:
	\begin{equation} \label{eq45}
		\begin{aligned}
			&\sum\limits_{t = {t_0}}^{{t_c}}\mathbb{E} {\left[ {{{\left\| {\nabla {{\cal L}_B}({{\hat \Theta }^t}) - \nabla {{\cal L}_B}({\Theta ^{{t_0}}})} \right\|}^2}} \right]} \\
			&\le 16{T^3}{\eta ^2}\sum\limits_{k = 1}^K {L_k^2} \parallel {\nabla _k}{\cal L}({\Theta ^{{t_0}}}){\parallel ^2}\\
			&\quad + 16{\eta ^2}{T^3}\sum\limits_{k = 1}^K {L_k^2\frac{{\pi _k^2}}{B}}\quad+64{T^3}\sum\limits_{k = 1}^K {M_k^2\Phi _k^2} ||{\rm E}_k^{{t_0}}||_{\cal F}^2.
		\end{aligned}
	\end{equation}
\end{lemma}

\begin{proof}
	\par By employing the chain rule and Taylor series expansion on \textbf{Lemma~\ref{lemma1}} in the work~\cite{Castiglia2022}, we initially prove the boundedness of causal representation error, i.e., \textbf{Theorem~\ref{theo:bounded}}.
	\par Combining the definition of causal representation error (Eq. (\ref{eq34}), we can derive the following equation:
	\begin{equation} \label{eq37}
		{\nabla _{_k}}{\cal L}_B^t({\bf{R}_k} + {\bf{E}}_k^t) := {\nabla _{_k}}{\cal L}_B^t({\bf{\hat R}}_k^s).
	\end{equation}
	Applying the chain rule to ${\nabla _{_k}}{{\cal L}_B}({\bf{\hat R}}_k^s)$, Eq. (\ref{eq37}) can be rewritten as:
	\begin{equation} \label{eq38}
		{\nabla _{k}}{\cal L}_B^t({\bf{\hat R}}_k^s) = {\nabla _{{\theta _k}}}{{\bf{R}_k} \cdot {\nabla _{k}}{\cal L}_B^t({{\bf{R}_k} + {\bf{E}}_k^t)}}.
	\end{equation}
	Next, we apply Taylor series expansion to rewrite Eq. (\ref{eq38}):
	\begin{equation} \label{eq39}
		\begin{aligned}
			{\nabla _{k}}{\cal L}_B^t({\bf{\hat R}}_k^s) &= {\nabla _{{\theta _k}}}{\bf{R}_k}[{\nabla _{{\bf{R}_k}}}{{\cal L}_B}({\bf{R}_k}) + R_{\rm top}^t({\bf{R}_k} + {\bf{E}}_k^t)]\\
			&= {\nabla _k}{\cal L}_B^t({\bf{R}_k}) + {\nabla _{{\theta _k}}}{\bf{R}_k} \cdot R_{\rm top}^t({\bf{R}_k} + {\bf{E}}_k^t).
		\end{aligned}
	\end{equation}
	By reordering the terms and incorporating the expected value and squared 2-norm, we can derive a more refined bound:
	\begin{equation} \label{eq40}
		\begin{aligned}
			&\mathbb{E}\parallel {\nabla _{_k}}{\cal L}_B^t({\bf{\hat R}}_k^s) - {\nabla _{_k}}{\cal L}_B^t({{\bf{\hat R}}_k}){\parallel ^2}\\
			&= \mathbb{E}\parallel {\nabla _{{\theta _k}}}{{\bf{\hat R}}_k} \cdot R_{top}^t({{\bf{\hat R}}_k} + {\bf{E}}_k^t){\parallel ^2}\\
			&\le M_k^2\Phi _k^2||{\bf{E}}_k^t|{|_{\cal F}}\\
			&= M_k^2\Phi _k^2\sum\limits_{j \ne k} {||\psi _j^t||_{\cal F}^2} \\
			&= M_k^2\Phi _k^2\sum\limits_{j \ne k} {\Psi _j^t}.
		\end{aligned}
	\end{equation}
	Based on the aforementioned proof, Theorem~\ref{theo:bounded} holds.
\end{proof}

\begin{proof}
	\par \textbf{Theorem~\ref{theo:conver}}: By Assumption~\ref{assump:bounded}, the predicted difference of any two communication rounds can be expressed as:
	\begin{equation} \label{eq47}
		\begin{aligned}
			&{\cal L}({\Theta ^{{t_c}}}) - {\cal L}({\Theta ^{{t_0}}})\\
			&\le \left\langle {\nabla {\cal L}({\Theta ^{{t_0}}}),{\Theta ^{{t_c}}} - {\Theta ^{{t_0}}}} \right\rangle  + \frac{L}{2}{\left\| {{\Theta ^{{t_c}}} - {\Theta ^{{t_0}}}} \right\|^2}\\
			&=  - \left\langle {\nabla {\cal L}({\Theta ^{{t_0}}}),\eta \sum\limits_{t = {t_0}}^{{t_c}} {\nabla {{\cal L}_B}({{\hat \Theta }^t})} } \right\rangle  + \frac{L}{2}{\left\| {\eta \sum\limits_{t = {t_0}}^{{t_c}} {\nabla {{\cal L}_B}({{\hat \Theta }^t})} } \right\|^2} \\
			&\le  - \eta \sum\limits_{t = {t_0}}^{{t_c}} {\left\langle {\nabla {\cal L}({\Theta ^{{t_0}}}),\nabla {{\cal L}_B}({{\hat \Theta }^t})} \right\rangle }  + \frac{{LT}}{2}{\eta ^2}\sum\limits_{t = {t_0}}^{{t_c}} {{{\left\| {\nabla {{\cal L}_B}({{\hat \Theta }^t})} \right\|}^2}} \\
			&\le - \eta \sum\limits_{t = {t_0}}^{{t_c}} {\left\langle {\nabla {\cal L}({\Theta ^{{t_0}}}),\nabla {{\cal L}_B}({{\hat \Theta }^t}) - \nabla {{\cal L}_B}({\Theta ^{{t_0}}})} \right\rangle } \\
			&- \eta \sum\limits_{t = {t_0}}^{{t_c}} {\left\langle {\nabla {\cal L}({\Theta ^{{t_0}}}),\nabla {{\cal L}_B}({\Theta ^{{t_0}}})} \right\rangle } \\
			& + \frac{{LT}}{2}{\eta ^2}\sum\limits_{t = {t_0}}^{{t_c}} {{{\left\| {\nabla {{\cal L}_B}({{\hat \Theta }^t}) - \nabla {{\cal L}_B}({\Theta ^{{t_0}}}) + \nabla {{\cal L}_B}({\Theta ^{{t_0}}})} \right\|}^2}} \\
			&\le \eta \sum\limits_{t = {t_0}}^{{t_c}} {\left\langle { - \nabla {\cal L}({\Theta ^{{t_0}}}),\nabla {{\cal L}_B}({{\hat \Theta }^t}) - \nabla {{\cal L}_B}({\Theta ^{{t_0}}})} \right\rangle } \\
			&- \eta \sum\limits_{t = {t_0}}^{{t_c}} {\left\langle {\nabla {\cal L}({\Theta ^{{t_0}}}),\nabla {{\cal L}_B}({\Theta ^{{t_0}}})} \right\rangle } \\
			&+ LT{\eta ^2}\sum\limits_{t = {t_0}}^{{t_c}} {{{\left\| {\nabla {{\cal L}_B}({{\hat \Theta }^t}) - \nabla {{\cal L}_B}({\Theta ^{{t_0}}})} \right\|}^2}} \\
			&+ LT{\eta ^2}\sum\limits_{t = {t_0}}^{{t_c}} {{{\left\| {\nabla {{\cal L}_B}({\Theta ^{{t_0}}})} \right\|}^2}} \\
			&\le \frac{\eta }{2}\sum\limits_{t = {t_0}}^{{t_c}} {{{\left\| {\nabla {\cal L}({\Theta ^{{t_0}}})} \right\|}^2}}  + \frac{\eta }{2}\sum\limits_{t = {t_0}}^{{t_c}} {{{\left\| {\nabla {{\cal L}_B}({{\hat \Theta }^t}) - \nabla {{\cal L}_B}({\Theta ^{{t_0}}})} \right\|}^2}} \\
			&- \eta \sum\limits_{t = {t_0}}^{{t_c}} {\left\langle {\nabla {\cal L}({\Theta ^{{t_0}}}),\nabla {{\cal L}_B}({\Theta ^{{t_0}}})} \right\rangle } \\
			& + LT{\eta ^2}\sum\limits_{t = {t_0}}^{{t_c}} {{{\left\| {\nabla {{\cal L}_B}({{\hat \Theta }^t}) - \nabla {{\cal L}_B}({\Theta ^{{t_0}}})} \right\|}^2}} \\
			&+ LT{\eta ^2}\sum\limits_{t = {t_0}}^{{t_c}} {{{\left\| {\nabla {{\cal L}_B}({\Theta ^{{t_0}}})} \right\|}^2}}.
		\end{aligned}
	\end{equation}
	Next, based on Assumptions~\ref{assump:unbiased}-\ref{assump:varia}, we take the mean of both sides of the Formula (\ref{eq47}):
	\begin{equation} \label{eq48}
		\begin{aligned}
			&\mathbb{E}[ {{\cal L}({\Theta ^{{t_c}}})}] - {\cal L}({\Theta ^{{t_0}}})\\
			&\le  - \frac{\eta }{2}\sum\limits_{t = {t_0}}^{{t_c}} {{{\left\| {\nabla {\cal L}({\Theta ^{{t_0}}})} \right\|}^2}}  + LT{\eta ^2}\sum\limits_{t = {t_0}}^{{t_c}} {{\mathbb{E}^{{t_0}}}\left[ {{{\left\| {\nabla {{\cal L}_B}({\Theta ^{{t_0}}})} \right\|}^2}} \right]} \\
			&+ \frac{\eta }{2}\sum\limits_{t = {t_0}}^{{t_c}} {(1 + 2LT\eta ){\mathbb{E}^{{t_0}}}\left[ {{{\left\| {\nabla {{\cal L}_B}({{\hat \Theta }^t}) - \nabla {{\cal L}_B}({\Theta ^{{t_0}}})} \right\|}^2}} \right]} \\
			&\le  - \frac{\eta }{2}\sum\limits_{t = {t_0}}^{{t_c}} {(1 - 2LT\eta ){{\left\| {\nabla {\cal L}({\Theta ^{{t_0}}})} \right\|}^2}}  + LT{\eta ^2}\sum\limits_{k = 1}^K {\frac{{\pi _k^2}}{B}} \sum\limits_{t = {t_0}}^{{t_c}} {{\eta ^2}} \\
			&+ \frac{\eta }{2}\sum\limits_{t = {t_0}}^{{t_c}} {(1 + 2LT\eta ){\mathbb{E}^{{t_0}}}\left[ {{{\left\| {\nabla {{\cal L}_B}({{\hat \Theta }^t}) - \nabla {{\cal L}_B}({\Theta ^{{t_0}}})} \right\|}^2}} \right]} \\
			&=  - \frac{T}{2}\eta (1 - 2LT\eta ){\left\| {\nabla {\cal L}({\Theta ^{{t_0}}})} \right\|^2} + LT{\eta ^2}\sum\limits_{k = 1}^K {\frac{{\pi _k^2}}{B}} \\
			&+ \frac{\eta }{2}\sum\limits_{t = {t_0}}^{{t_c}} {(1 + 2LT\eta ){\mathbb{E}^{{t_0}}}\left[ {{{\left\| {\nabla {{\cal L}_B}({{\hat \Theta }^t}) - \nabla {{\cal L}_B}({\Theta ^{{t_0}}})} \right\|}^2}} \right]}.
		\end{aligned}
	\end{equation}
	By introducing Lemma~\ref{lemma2}, we can further deduce:
	\begin{equation} \label{eq49}
		\begin{aligned}
			&\mathbb{E}\left[ {{\cal L}({\Theta ^{{t_c}}})} \right] - {\cal L}({\Theta ^{{t_0}}})\\
			&=  - \frac{T}{2}\eta (1 - 2LT\eta ){\left\| {\nabla {\cal L}({\Theta ^{{t_0}}})} \right\|^2}\\
			&+ 8{T^3}{\eta ^3}(1 + 2LT\eta )\sum\limits_{k = 1}^K {L_k^2||{\nabla _k}{\cal L}({\Theta ^{{t_0}}})|{|^2}} \\
			&+ 8{T^3}{\eta ^3}(1 + 2LT\eta )\sum\limits_{k = 1}^K {L_k^2\frac{{\pi _k^2}}{B}} \\
			&+ 32{T^3}\eta (1 + 2LT\eta )\sum\limits_{k = 1}^K {M_k^2\Phi _k^2||{\bf{E}}_k^t||_{\cal F}^2}  + L{T^2}{\eta ^2}\sum\limits_{k = 1}^K {\frac{{\pi _k^2}}{B}} \\
			&\le  - \frac{T}{2}\sum\limits_{k = 1}^K {\eta (1 - 2LT\eta  - 16{T^2}L_k^2{\eta ^2} - 16{T^3}L_k^2{\eta ^3})||{\nabla _k}{\cal L}({\Theta ^{{t_0}}})|{|^2}} \\
			&+ (L{T^2}{\eta ^2} + 8{T^3}L_k^2{\eta ^3} + 8{T^4}LL_k^2{\eta ^4})\sum\limits_{k = 1}^K {\frac{{\pi _k^2}}{B}} \\
			&+ 32{T^3}\eta (1 + 2LT\eta )\sum\limits_{k = 1}^K {M_k^2\Phi _k^2||{\bf{E}}_k^t||_{\cal F}^2} .
		\end{aligned}
	\end{equation}
	Next, taking into account the conditional limitations in Theorem~\ref{theo:conver}, i.e., $\eta  \le \frac{1}{{16T\max \{ L,{{\max }_k},{L_k}\} }}$. The Formula (\ref{eq49}) can therefore be rewritten as follows:
	\begin{equation} \label{eq50}
		\begin{aligned}
			&\mathbb{E}\left[ {{\cal L}({\Theta ^{{t_c}}})} \right] - {\cal L}({\Theta ^{{t_0}}})\\
			&=  - \frac{T}{2}\sum\limits_{k = 1}^K {\eta (1 - \frac{1}{8} - \frac{1}{{16}} - \frac{1}{{{{16}^2}}}){{\left\| {{\nabla _k}{\cal L}({\Theta ^{{t_0}}})} \right\|}^2}} \\
			&+ (L{T^2}{\eta ^2} + 8{T^3}L_k^2{\eta ^3} + 8{T^4}LL_k^2{\eta ^4})\sum\limits_{k = 1}^K {\frac{{\pi _k^2}}{B}} \\
			&+ 16{T^3}\eta (1 + 2LT\eta )\sum\limits_{k = 1}^K {M_k^2\Phi _k^2||{\bf E}_k^t||_{\cal F}^2} \\
			&\le  - \frac{{3T}}{8}\eta {\left\| {\nabla {\cal L}({\Theta ^{{t_0}}})} \right\|^2}\\
			&+ (L{T^2}{\eta ^2} + 8{T^3}L_k^2{\eta ^3} + 8{T^4}LL_k^2{\eta ^4})\sum\limits_{k = 1}^K {\frac{{\pi _k^2}}{B}} \\
			&+ 32{T^3}\eta (1 + 2LT\eta )\sum\limits_{k = 1}^K {M_k^2\Phi _k^2||{\bf E}_k^t||_{\cal F}^2} .
		\end{aligned}
	\end{equation}
	We rewrite the $\eta {\left\| {\nabla {\cal L}({\Theta ^{{t_0}}})} \right\|^2}$ term as follows:
	\begin{equation} \label{eq51}
		\begin{aligned}
			& \eta {\left\| {\nabla {\cal L}({\Theta ^t})} \right\|^2}\\
			& \le \frac{8}{3}(LT{\eta ^2} + 8{T^2}L_k^2{\eta ^3} + 8{T^3}LL_k^2{\eta ^4})\sum\limits_{k = 1}^K {\frac{{\pi _k^2}}{B}} \\
			& + 86{T^2}\eta (1 + 2LT\eta )\sum\limits_{k = 1}^K {M_k^2\Phi _k^2||{\bf E}_k^t||_{\cal F}^2} .
		\end{aligned}
	\end{equation}
	We consider the total expectation of inequality~(\ref{eq51}) across $T$ iterations:
	\begin{equation} \label{eq52}
		\begin{aligned}
			& \eta \mathbb{E}\left[ {{{\left\| {\nabla {\cal L}({\Theta ^t})} \right\|}^2}} \right] \\
			& \le \frac{{4\left[ {{\cal L}({\Theta ^0}) - \mathbb{E}\left[ {{\cal L}({\Theta ^T})} \right]} \right]}}{T} \\
			& \quad + \frac{8}{3}(LT{\eta ^2} + 8{T^2}L_k^2{\eta ^3} + 8{T^3}LL_k^2{\eta ^4})\sum\limits_{k = 1}^K {\frac{{\pi _k^2}}{B}} \\
			& \quad + 86{T^2}\eta (1 + 2LT\eta )\sum\limits_{k = 1}^K {M_k^2\Phi _k^2\left[ {||{\bf E}_k^t||_{\cal F}^2} \right]} .
		\end{aligned}
	\end{equation}
	Recalling Theorem~\ref{theo:bounded}, the scope of the Formula (\ref{eq52}) can be limited as:
	
	\begin{equation} \label{eq53}
		\begin{aligned}
			& \eta \mathbb{E}\left[ {{{\left\| {\nabla {\cal L}({\Theta ^t})} \right\|}^2}} \right] \\
			& \le \frac{{4\left[ {{\cal L}({\Theta ^0}) - \left[ {{\cal L}({\Theta ^T})} \right]} \right]}}{T} \\
			& \quad + \frac{8}{3}\sum\limits_{t = 1}^T {(LT{\eta ^2} + 8{T^2}L_k^2{\eta ^3} + 8{T^3}LL_k^2{\eta ^4})\sum\limits_{k = 1}^K {\frac{{\pi _k^2}}{B}} } \\
			& \quad + 86{T^2}\sum\limits_{t = 1}^T {\eta (1 + 2LT\eta )} \sum\limits_{k = 1}^K {M_k^2\Phi _k^2\sum\limits_{j \ne k} {\Psi _j^t} } .
		\end{aligned}
	\end{equation}
	Averaging over $T$ global rounds, we can have:
	\begin{equation} \label{eq56}
		\begin{aligned}
			&\mathbb{E}{^T}\left[ {{{\left\| {\nabla {\cal L}({\Theta ^t})} \right\|}^2}} \right] \\
			&\le \frac{{4\left[ {{\cal L}({\Theta ^0}) - \left[ {{\cal L}({\Theta ^T})} \right]} \right]}}{{T\eta }} \\
			&\quad+ \frac{8}{3}\sum\limits_{k = 1}^K {\frac{{\pi _k^2}}{B}(LT\eta  + 8{T^2}L_k^2{\eta ^2} + 8{T^3}LL_k^2{\eta ^3})} \\
			&\quad+ 86{T^2}(1 + 2LT\eta )\sum\limits_{k = 1}^K {M_k^2\Phi _k^2\sum\limits_{j \ne k} {\Psi _j^t} } \\
			&\le \frac{{4\left[ {{\cal L}({\Theta ^0}) - \left[ {{\cal L}({\Theta ^T})} \right]} \right]}}{{\eta T}} + 6\eta TL\sum\limits_{k = 1}^K {\frac{{\pi _k^2}}{B}} \\
			&\quad+ 92T\sum\limits_{k = 1}^K {\sum\limits_{t = 0}^{T - 1} {M_k^2\Phi _k^2} } \sum\limits_{j = 1,j \ne k}^K {\Psi _j^t}.
		\end{aligned}
	\end{equation}
	By combining the Lemma~\ref{lemma1} and Lemma~\ref{lemma2} of the work~\cite{Castiglia2022}, we have the proof of Theorem~\ref{theo:conver}.
\end{proof}

\end{document}